\documentclass[letterpaper,11pt]{article}

\usepackage{amsmath}
\usepackage{amsfonts}
\usepackage{amssymb}
\usepackage{amsthm}
\usepackage{url,ifthen}
\usepackage{srcltx}
\usepackage{multirow}
\usepackage{boxedminipage}
\usepackage[margin=1.0in]{geometry}
\usepackage{nicefrac}
\usepackage{xspace}
\usepackage{graphicx}
\usepackage{xcolor}
\definecolor{DarkGreen}{rgb}{0.1,0.5,0.1}
\definecolor{DarkRed}{rgb}{0.5,0.1,0.1}
\definecolor{DarkBlue}{rgb}{0.1,0.1,0.5}
\usepackage{tikz}
\usetikzlibrary{shapes,arrows}

\usepackage[small]{caption}

\usepackage[pdftex]{hyperref}
\hypersetup{
    unicode=false,          
    pdftoolbar=true,        
    pdfmenubar=true,        
    pdffitwindow=false,      
    pdfnewwindow=true,      
    colorlinks=true,       
    linkcolor=DarkBlue,          
    citecolor=DarkGreen,        
    filecolor=DarkGreen,      
    urlcolor=DarkBlue,          
    pdftitle={},
    pdfauthor={},    
}

\def\draft{1} 

\def\submit{0} 

\ifnum\draft=1 
    \def\ShowAuthNotes{1}
\else
    \def\ShowAuthNotes{0}
\fi

\ifnum\submit=1
\newcommand{\forsubmit}[1]{#1}
\newcommand{\forreals}[1]{}
\else
\newcommand{\forreals}[1]{#1}
\newcommand{\forsubmit}[1]{}
\fi

\ifnum\ShowAuthNotes=1
\newcommand{\authnote}[2]{{ \footnotesize \bf{\color{DarkRed}[#1's Note:
{\color{DarkBlue}#2}]}}}
\else
\newcommand{\authnote}[2]{}
\fi

\newtheorem{theorem}{Theorem}[section]
\newtheorem{remark}[theorem]{Remark}
\newtheorem{lemma}[theorem]{Lemma}
\newtheorem{corollary}[theorem]{Corollary}
\newtheorem{proposition}[theorem]{Proposition}
\newtheorem{claim}[theorem]{Claim}

\theoremstyle{definition}
\newtheorem{definition}[theorem]{Definition}

\newcommand{\chapterref}[1]{\hyperref[ch:#1]{Chapter~\ref{ch:#1}}}
\newcommand{\claimlabel}[1]{\label{claim:#1}}
\newcommand{\claimref}[1]{\hyperref[claim:#1]{Claim~\ref{claim:#1}}}
\newcommand{\corollarylabel}[1]{\label{cor:#1}}
\newcommand{\corollaryref}[1]{\hyperref[cor:#1]{Corollary~\ref{cor:#1}}}
\newcommand{\definitionlabel}[1]{\label{def:#1}}
\newcommand{\definitionref}[1]{\hyperref[def:#1]{Definition~\ref{def:#1}}}
\newcommand{\equationlabel}[1]{\label{eq:#1}}
\newcommand{\equationref}[1]{\hyperref[eq:#1]{Equation~\ref{eq:#1}}}
\renewcommand{\eqref}[1]{\hyperref[eq:#1]{(\ref{eq:#1})}}

\newcommand{\factref}[1]{\hyperref[fact:#1]{Fact~\ref{fact:#1}}}
\newcommand{\figurelabel}[1]{\label{fig:#1}}
\newcommand{\figureref}[1]{\hyperref[fig:#1]{Figure~\ref{fig:#1}}}
\newcommand{\itemlabel}[1]{\label{item:#1}}
\newcommand{\itemref}[1]{\hyperref[item:#1]{Item~(\ref{item:#1})}}
\newcommand{\lemmalabel}[1]{\label{lem:#1}}
\newcommand{\lemmaref}[1]{\hyperref[lem:#1]{Lemma~\ref{lem:#1}}}

\newcommand{\propref}[1]{\hyperref[prop:#1]{Proposition~\ref{prop:#1}}}
\newcommand{\propositionlabel}[1]{\label{prop:#1}}
\newcommand{\propositionref}[1]{\hyperref[prop:#1]{Proposition~\ref{prop:#1}}}

\newcommand{\remarkref}[1]{\hyperref[rem:#1]{Remark~\ref{rem:#1}}}
\newcommand{\sectionlabel}[1]{\label{sec:#1}}
\newcommand{\sectionref}[1]{\hyperref[sec:#1]{Section~\ref{sec:#1}}}
\newcommand{\theoremlabel}[1]{\label{thm:#1}}
\newcommand{\theoremref}[1]{\hyperref[thm:#1]{Theorem~\ref{thm:#1}}}

\usepackage[T1]{fontenc}
\usepackage{kpfonts}
\usepackage{microtype}

\newcommand{\Esymb}{\mathbb{E}}
\newcommand{\Psymb}{\mathbb{P}}

\DeclareMathOperator*{\E}{\Esymb}

\DeclareMathOperator*{\ProbOp}{\Psymb}
\renewcommand{\Pr}{\ProbOp}

\newcommand{\widebar}[1]{\overline{#1}}

\newcommand{\mper}{\,.}
\newcommand{\mcom}{\,,}

\renewcommand{\hat}{\widehat}

\newcommand{\cR}{{\cal R}}

\newcommand{\cX}{{\cal X}}
\newcommand{\cY}{{\cal Y}}

\newcommand{\defeq}{\,\stackrel{{\mathrm{def}}}{=}\,}
\renewcommand{\leq}{\leqslant}
\renewcommand{\le}{\leqslant}

\renewcommand{\ge}{\geqslant}

\newcommand{\Set}[1]{\left\{#1\right\}}

\newcommand{\Norm}[1]{\left\lVert#1\right\rVert}

\newcommand{\R}{\mathbb{R}}

\usepackage{bm}

\newcommand{\poly}{{\rm poly}}

\renewcommand{\epsilon}{\varepsilon}

\newcommand{\eps}{\epsilon}

\newcommand{\remove}[1]{}

\newcommand{\range}{{\cal R}}

\newcommand{\ignore}[1]{}

\newcommand{\trans}{\top}
\newcommand{\NSI}{{\sc NSI}\xspace}
\newcommand{\AltLS}{{\sc LS}\xspace}
\newcommand{\MedianLS}{{\sc MedianLS}\xspace}

\newcommand{\SAltLS}{{\sc SAltLS}\xspace}
\newcommand{\Split}{{\sc Split}\xspace}

\newcommand{\Init}{{\sc Initialize}\xspace}
\newcommand{\GS}{{\sc QR}\xspace}
\newcommand{\QR}{{\sc QR}\xspace}
\newcommand{\SmoothGS}{{\sc SmoothQR}\xspace}
\newcommand{\SmoothQR}{{\sc SmoothQR}\xspace}

\renewcommand{\tilde}{\widetilde}
\newcommand{\mustar}{\mu^*}
\renewcommand{\bar}{\widebar}

\usepackage{enumitem}
{\end{itemize}}

\newenvironment{enum}
{
\begin{enumerate}[noitemsep,topsep=0pt,parsep=0pt,partopsep=0pt]}
{\end{enumerate}}

\title{Understanding Alternating Minimization\\ for Matrix Completion}

\author{Moritz Hardt\thanks{IBM Research Almaden. Email: {\tt
mhardt@us.ibm.com}}}

\begin{document}
\maketitle
\begin{abstract}
Alternating minimization is a widely used and empirically successful heuristic
for matrix completion and related low-rank optimization problems. Theoretical
guarantees for alternating minimization have been hard to come by and are
still poorly understood. This is in part because the heuristic is iterative
and non-convex in nature. We give a new algorithm based on alternating
minimization that provably recovers an unknown low-rank matrix from a random
subsample of its entries under a standard incoherence assumption.  Our results
reduce the sample size requirements of the alternating minimization approach
by at least a quartic factor in the rank and the condition number of the
unknown matrix. These improvements apply even if the matrix is only close to
low-rank in the Frobenius norm. Our algorithm runs in nearly linear time in
the dimension of the matrix and, in a broad range of parameters, gives the
strongest sample bounds among all subquadratic time algorithms that we are aware
of. 

Underlying our work is a new robust convergence analysis of the well-known
Power Method for computing the dominant singular vectors of a matrix.  This
viewpoint leads to a conceptually simple understanding of alternating
minimization. In addition, we contribute a new technique for controlling the
coherence of intermediate solutions arising in iterative algorithms based on a
smoothed analysis of the QR factorization. These techniques may be of interest
beyond their application here.
\end{abstract}

\vfill
\thispagestyle{empty}
\pagebreak

%
%

\section{Introduction}

Alternating minimization is an empirically successful heuristic for the matrix
completion problem in which the goal is to recover an unknown low-rank matrix
from a subsample of its entries. Matrix completion has received a tremendous
amount of attention over the past few years due to its fundamental role as an
optimization problem and its applicability in number of areas including
collaborative filtering and quantum tomography. Alternating minimization has
been used early on in the context of matrix completion~\cite{BellK07,HaldarH09} 
and continues to play an important role in practical approaches to the problem. 
The approach also formed an important component in the winning submission
for the Netflix Prize~\cite{KorenBV09}.

Given a subset~$\Omega$ of entries drawn from an unknown matrix~$A,$
Alternating minimization starts from a poor approximation $X_0Y_0^\trans$ to
the target matrix and gradually improves the approximation quality by fixing one of
the factors and minimizing a certain objective over the other factor. Here,
$X_0,Y_0$ each have $k$ columns where $k$ is the target rank of the
factorization. The least squares objective is the typical choice. In this case, at
step~$\ell$ we solve the optimization problem
\[
X_\ell = \arg\min_X\sum_{(i,j)\in\Omega} 
\left[A_{ij} - (XY_{\ell-1}^\trans)_{ij}\right]^2.
\]
This optimization step is then repeated with $X_\ell$ fixed in order to
determine $Y_\ell$ as 
\[
Y_\ell = \arg\min_X\sum_{(i,j)\in\Omega} 
\left[A_{ij} - (X_\ell Y^\trans)_{ij}\right]^2.
\]
Separating the factors~$X_\ell$ and $Y_\ell$ is what makes the optimization
step tractable. This basic update step is usually combined with an initialization
procedure for finding~$X_0,Y_0,$ as well as methods for modifying
intermediate solutions, e.g., truncating large entries. More than a specific
algorithm we think of alternating minimization as a framework for
solving a non-convex low-rank optimization problem.

A major advantage of alternating minimization over alternatives is that each
update is computationally cheap and has a small memory footprint as we only
need to keep track of $2k$~vectors. In contrast,
the \emph{nuclear norm} approach to matrix
completion~\cite{CandesR09,Recht11,CandesT10} requires solving a semidefinite
program. The advantage of the nuclear norm approach is that it comes with 
strong theoretical guarantees under certain assumptions on the unknown matrix and
the subsample of its entries. There are two (by now standard) assumptions
which together imply that nuclear norm minimization succeeds. The first is
that the subsample~$\Omega$ includes each entry of~$A$ uniformly at random
with probability~$p.$ The second assumption is that the first $k$ singular
vectors of~$A$ span an \emph{incoherent} subspace. Informally coherence
measures the correlation of the subspace with any standard basis vector. More
formally, the coherence of a $k$-dimensional subspace of $\R^n$ is at most
$\mu$ if the projection of each standard basis vector has norm at most
$\sqrt{\mu k / n}.$ The space spanned by the top~$k$ singular space of various 
random matrix models typically satisfies this property with small~$\mu.$ 
But also real-world matrices tend to exhibit incoherence when $k$ is 
reasonably small.

Theoretical results on matrix completion primarily apply to the \emph{nuclear
norm}
semidefinite program which is prohibitive to execute on realistic instance sizes. 
There certainly has been progress on practical algorithms for
solving related convex programs~\cite{JiY09,MazumderHT10,JaggiS10,AvronKKS12,HsiehO14}.
Unfortunately, these algorithms are not known to achieve the same type of
recovery guarantees attained by exact nuclear norm minimization. This raises
the important question if there are fast algorithms for matrix completion that
come with guarantees on the required sample size comparable to those achieved by
nuclear norm minimization. In this work we make progress on this problem by proving
strong sample complexity bounds for alternating minimization. Along the way
our work helps to give a theoretical justification and understanding for why
alternating minimization works.

\subsection{Our results}

We begin with our result on the \emph{exact} matrix completion problem where
the goal is to recover an unknown rank~$k$ matrix~$M$ from a
subsample~$\Omega$ of its entries where each entry is included independently
with probability~$p.$ Here and in the following we will always assume
that~$M=U\Lambda U^\trans$ is a 
symmetric $n\times n$ matrix with singular values $\sigma_1\ge\dots\ge
\sigma_k.$ Our result generalizes straightforwardly to rectangular matrices 
as we will see.

Our algorithm will output a pair of matrices $(X,Y)$ where $X$ is an
orthonormal $n\times k$ matrix that approximates $U$ in the strong sense that
$\|(I-UU^\trans)X\|\le \epsilon.$ Moreover, the matrix $XY^\trans$ is close
to~$M$ in Frobenius norm. To state the theorem we formally define the coherence 
of~$U$ as $\mu(U) \defeq \max_{i\in[n]}(n/k)\|e_i^\trans U\|_2^2$
where $e_i$ is the $i$-th standard basis vector.

\begin{theorem}\theoremlabel{intro-1}
Given a sample of size $\tilde O(pn^2)$ drawn from an
unknown $n\times n$ matrix $M=U\Lambda U^\trans$ of rank~$k$ by including each
entry with probability~$p,$ our algorithm outputs with high probability a pair
of matrices~$(X,Y)$ such that $\|(I-UU^\trans) X\|\le\epsilon$ 
and $\|M-XY^\trans\|_F\le\epsilon\|M\|_F$ provided that
\begin{equation}\equationlabel{exact-result}
pn \ge k(k+\log(n/\epsilon))\mu(U)\left(\|M\|_F/\sigma_k\right)^2\mper
\end{equation}
\end{theorem}
Our result should be compared with two remarkable recent works by Jain, Netrapalli and
Sanghavi~\cite{JainNS13} and Keshavan~\cite{Keshavan12} who gave 
rigorous sample complexity bounds for alternating minimization. 
\cite{JainNS13}~obtained the bound $pn 
\ge k^{7}(\sigma_1/\sigma_k)^6\mu(U)^2$ 
and Keshavan obtained the incomparable bound $pn 
\ge k(\sigma_1/\sigma_k)^8\mu(U)$ that is superior when the matrix has small
\emph{condition number}~$\sigma_1/\sigma_k.$
Since $\|M\|_F\le\sqrt{k}\sigma_1$ our result
improves upon~\cite{JainNS13} by at least a factor of
$k^4(\sigma_1/\sigma_k)^4\mu(U)$ and improves on~\cite{Keshavan12} as soon as
$\sigma_1/\sigma_k\gg k^{1/3}.$ The improvement is larger when $\|M\|_F =
O(\sigma_1)$ which we expect if the singular values decay rapidly.

\theoremref{intro-1} is a special case of \theoremref{main}.
We remark that the number of least squares update steps is bounded by
$O(\log(n/\epsilon)\log n).$ The cost of performing these update steps is up to a
logarithmic factor what dominates
the worst-case running time of our algorithm. It can be seen that the least squares
problem can be solved in time $O\left(nk^3+|\Omega|\cdot k\right)$ which is is
linear in~$n+|\Omega|$ and
polynomial in~$k.$  The number of update steps
enters the sample complexity since we assume (as in previous work) that fresh
samples are used in each step. However, the logarithmic dependence
on~$1/\epsilon$ guarantees exponentially fast convergence and allows us to
obtain any inverse polynomial error with only a constant factor overhead in
sample complexity.

\paragraph{Noisy matrix completion.}

In noisy matrix completion the
unknown matrix is only close to low-rank, typically in Frobenius norm. 
Our results apply to any matrix of the form $A = M + N,$ where $M=U\Lambda
U^\trans$ is a matrix of rank~$k$ as before and $N=(I-UU^\trans)A$ is the part
of $A$ not captured by the dominant singular vectors. Here, $N$ can be an
arbitrary deterministic matrix
that satisfies the following constraints:
\begin{equation}\equationlabel{noise-bound}
\max_{i\in[n]}\|e_i^\trans N\|^2\le \frac{\mu_N}n\cdot\sigma_k^2
\quad\text{and}\quad
\max_{ij\in[n]}|N_{ij}|\le \frac{\mu_N}n\cdot\Norm{A}_F\mper
\end{equation}
Here, $e_i$ denotes the $i$-th standard basis vector so that $\|e_i^\trans N\|$ is
the Euclidean norm of the $i$-th row of $N.$ 
The conditions state no entry and no row of $N$ should be too large compared to 
the Frobenius norm of~$N.$ We can think of the parameter $\mu_N$ as an
analog to the coherence parameter $\mu(U)$ that we saw earlier. Since $N$
could be close to full rank, denoting by $V$ the space spanned by the columns
of $N,$ the parameter $\mu(V)$ is no longer meaningful.
If the rank of~$V$ is $k,$ then our assumptions roughly reduce to 
what is implied by requiring $\mu(V)\le\mu_N.$

From here on we let $\mustar = \max\Set{\mu(U),\mu_N,\log n}.$ We have the following theorem.
\begin{theorem}
\theoremlabel{noisy}
Given a sample of size $\tilde O(pn^2)$ drawn from an
unknown $n\times n$ matrix $A = M+N$ where $M=U\Lambda U^\trans$ 
has rank~$k$ and $N=(I-UU^\trans)A$ satisfies~\eqref{noise-bound},
our algorithm outputs with high probability
$(X,Y)$ such that $\|(I-UU^\trans) X\|\le\epsilon$ 
and $\|M-XY^\trans\|_F\le\epsilon\|A\|_F$ provided that
\begin{equation}\equationlabel{noisy-result}
pn\ge k(k+\log(n/\epsilon))\mustar
\left(\frac{\|M\|_F+\|N\|_F/\epsilon}{\sigma_k}\right)^2
\Big(1-\frac{\sigma_{k+1}}{\sigma_k}\Big)^{-5} \mper
\end{equation}
\end{theorem}
The theorem is a strict generalization of
the noise-free case which we recover by setting $N=0$ in which case
the separation parameter $\gamma_k:=1-\sigma_{k+1}/\sigma_k$ is equal to~$1.$ 
The result follows from \theoremref{main} that gives a somewhat stronger
sample complexity bound. Compared to our noise-free bound, there are two new
parameters that enter the sample complexity. The first one is the separation
parameter~$\gamma_k.$ The second is the
quantity~$\|N\|_F/\epsilon.$ To interpret this quantity, suppose that  that
$A$ has a good low-rank approximation in Frobenius norm, formally, $\|N\|_F\le
\epsilon\|A\|_F$ for $\epsilon\le 1/2,$ then it must also be the case that
$\|N\|_F/\epsilon\le 2\|M\|_F.$ Our algorithm then finds a good
rank~$k$ approximation with at most $\tilde{O}(k^3(\sigma_1/\sigma_k)^2\mustar n)$ samples
assuming $\gamma_k=\Omega(1).$ Hence, assuming that $A$
has a good rank~$k$ approximation in Frobenius norm and that $\sigma_k$ and
$\sigma_{k+1}$ are well-separated, our bound recovers the
noise-free bound up to a constant factor. 

Note that if we're only interested in the second error bound
$\|M-XY^\trans\|_F\le \epsilon\|M\|_F + \|N\|_F,$ we we can eliminate the
dependence on the condition number in the sample complexity entirely. The
reason is that any singular value smaller than $\epsilon\sigma_1/k$ can be
treated as part of the noise matrix. Assuming the condition number is at
least~$k$ to begin with we can always find two singular values that have
separation at least~$\Omega(k).$ This ensures that the sample requirement is
polynomial in~$k$ without any dependence on the condition number and gives us
the following corollary.
\begin{corollary}\corollarylabel{noisy}
Under the assumptions of \theoremref{noisy}, if $\sigma_1\ge
k\sigma_k/\epsilon,$ then we can find $X,Y$ such
that $\|M-XY^\trans\|_F\le\epsilon\|A\|_F$ provided that 
$pn\ge \poly(k)\mustar.$
\end{corollary}
The previous corollary is remarkable, because small error in Frobenius norm is
the most common error measure in the literature on matrix completion. The
result shows that in this error measure, there is no dependence on the
condition number. The result is tight for $k=O(1)$ up to constant factors even
information-theoretically as we will discuss below. 

The approach of Jain et al.~was adapted to the noisy setting by Gunasekar et
al.~\cite{GunasekarAGG13} showing roughly same sample complexity 
in the noisy setting under some assumptions on the noise matrix. 
We achieve the same improvements
over~\cite{GunasekarAGG13} as we did compared to~\cite{JainNS13} in the
noise-free case. Moreover, our assumptions in \eqref{noise-bound} 
are substantially weaker than the assumption of \cite{GunasekarAGG13}. The
latter work required the largest entry of $N$ in absolute value 
to be bounded by~$O(\sigma_k/n\sqrt{k}).$ This directly implies that each row
of $N$ has norm at most $O(\sigma_k/\sqrt{kn})$ and that $\|N\|_F\le
O(\sigma_k/\sqrt{k}).$ 
Moreover under this assumption we would have
$\gamma_k\ge1-o_k(1).$ Keshavan's result~\cite{Keshavan12} also applies to the 
noisy setting, but it requires $\|N\|\le
(\sigma_k/\sigma_1)^3$ and $\max_i\|e_i^\trans N\|\le
\sqrt{\mu(U)k/n}\|N\|.$ In particular this bound does not allow
$\|N\|_F$ to grow with $\|M\|_F.$ Since neither result allows arbitrarily
small singular value separation, we cannot use these results to eliminate the
dependence on the condition number as is possible using our technique.

\paragraph{Remark on required sample complexity and assumptions.}
It is known that information-theoretically $\Omega(k\mu(U)n)$ measurements are
necessary to recover the unknown matrix~\cite{CandesT10} and this bound is achieved (up to
log-factors) by the nuclear norm semidefinite program. Compared with the
information-theoretic optimum our bound suffers a factor
$O(k(\|M\|_F/\sigma_k)^2)$ loss. While we do not know if this loss is
necessary, there is a natural barrier. If we denote by $P_\Omega(A)$ the
matrix in which all unobserved entries are~$0$ and the others are scaled by
$1/p,$ then 
$\Omega(k\mu(U)(\|M\|_F/\sigma_k)^2n)$ samples are necessary to ensure that
$P_\Omega(A)$ preserves the $k$-th singular value to within constant relative
error. Formally, $\|P_\Omega(A) - A\|_2 \le 0.1\sigma_k.$ While this is
not a necessary requirement for alternating least squares, it represents
the current bottleneck for finding a good initial matrix.

It is also known that without an incoherence assumption the matrix completion
problem can be ill-posed and recovery becomes infeasible even
information-theoretically~\cite{CandesT10}. Moreover, even on
incoherent matrices it was recently shown that already the exact matrix
completion problem remains computationally hard to approximate in a strong
sense~\cite{HardtMRW14}. This shows that additional assumptions 
are needed beyond incoherence to make the problem tractable.  

\section{Proof overview and techniques}

\paragraph{Robust convergence of subspace iteration.}
An important observation of \cite{JainNS13} is that the update rule in
alternating minimization can be analyzed as a noisy update step of the well
known \emph{power method} for computing eigenvectors, also called
\emph{subspace iteration} when applied to multiple vectors simultaneously.
The noise term that arises depends on the sampling error induced by the
subsample of the entries. We further develop this point of view by giving a
new robust convergence analysis of the power method. 

To illustrate the technique, consider a model of numerical linear algebra in
which an input matrix~$A$ can only be accessed through noisy matrix vector
products of the form $Ax+g,$ where $x$ is a chosen vector and $g$ is a
possibly adversarial noise term. 
Our goal is to compute the dominant singular vectors $u_1,\dots,u_k$ of the
matrix~$A.$ Subspace
iteration starts with an initial guess, an orthonormal matrix
$X_0\in\R^{n\times k}$ typically chosen at random. 
The algorithm then repeatedly 
computes $Y_\ell = AX_{\ell-1} + G_\ell,$ followed by an orthonormalization
step in order to obtain $X_\ell$ from $Y_\ell.$ Here, $G_\ell$ is the noise
variable added to the computation.

\theoremref{convergence} characterizes the convergence behavior of this
general algorithm. An important component of our analysis is the choice of a suitable
potential function that decreases at each step. Here we make use of the
tangent of the \emph{largest principal angle} between the subspace~$U$ spanned by the
first $k$ singular vectors of the input matrix and the $k$-dimensional space
spanned by the columns of the iterate~$X_\ell.$ Principal angles are a very
useful tool in numerical analysis that we briefly recap in \sectionref{nsi}. 
Our analysis shows that the algorithm
essentially converges at the rate of $(\sigma_{k+1}+\Delta)/(\sigma_k-\Delta)$
for some $\Delta \ll \sigma_k$ under suitable
conditions on the noise matrix~$G_\ell.$ 

\paragraph{Least squares update.}
The least squares update works as follows:
\begin{equation}\equationlabel{altls-update}
Y_\ell = \arg\min_Y \|P_{\Omega}(A - X_{\ell-1}Y^\trans)\|_F^2\mper
\end{equation}
Since we can focus on symmetric matrices without loss of generality, there is
no need for an alternating update in which the left and right factor are
flipped. We therefore drop the term ``alternating''.
We can express the optimal~$Y_\ell$ as $Y_\ell=AX_{\ell-1}+G_\ell$ using gradient
information about the least squares objective. The error term $G_\ell$ has an
intriguing property. Its norm~$\|G_\ell\|$ depends on the quantity $\Norm{V^\trans
X_{\ell-1}}$ which coincides with the sine of the largest principal angle
between $U$ and $X_{\ell-1}.$  This property ensures that as the algorithm begins
to converge the norm of the error term starts to diminish. Near exact recovery
is now possible (assuming the matrix has rank at most~$k$). A novelty in our
approach is that we obtain strong bounds on $\|G_\ell\|$ by computing $O(\log n)$ 
independent copies of $Y_\ell$ (using fresh samples) and taking the componentwise 
median of the resulting matrices. The resulting procedure called \MedianLS is
analyzed in \sectionref{altls}.

A difficulty with iterating the least squares update in general is that it is unclear
how well it converges from a random initial matrix $X_0.$ In our analysis we
therefore use an initialization procedure that finds a matrix $X_0$ that
satisfies $\Norm{V^\trans X_0}\le 1/4.$ Our initialization procedure is based
on (approximately) computing the first~$k$ singular vectors of
$P_{\Omega}(A).$ To rule out large entries in the vectors we truncate the
resulting vectors. While this general approach is standard, our truncation
procedure first applies a random rotation to the vectors that leads to a
tighter analysis than the naive approach.

\paragraph{Smooth orthonormalization.}
A key novelty in our approach is the way we argue about the coherence of each
iterate $X_\ell.$ Ideally, we would like to argue that $\mu(X_\ell)=
O(\mustar).$ A direct approach would be to argue that $X_\ell$ was
obtained from $Y_\ell$ using the QR-factorization and so $X_\ell = Y_\ell
R^{-1}$ for some invertible $R.$ This gives the bound 
$\|e_i^\trans X_\ell\|\le\|e_i^\trans Y_\ell\|\cdot\|R^{-1}\|$ that
unfortunately is quite lossy and leads to a dependence on the condition number.

We avoid this problem using an idea that's closely related to the
\emph{smoothed analysis} of the QR-factorization.  Sankar, Spielman and
Teng~\cite{SankarST06} showed that while the perturbation stability of QR can
be quadratic, it is constant after adding a sufficiently large amount of
Gaussian noise. In the context of smoothed analysis this is usually
interpreted as saying that there are ``few bad inputs'' for the QR
factorization. In our context, the matrix $Y_\ell$ is already the outcome of a
noisy operation $Y_\ell=AX_{\ell-1}+G_\ell$ and so there is no harm in
actually adding a Gaussian noise matrix~$H_\ell$ to $Y_\ell$ provided that the norm
of that matrix is no larger than that of $G_\ell.$ Roughly speaking, this
will allow us to argue that there is no dependence on the condition number when applying the
QR-factorization to $Y_\ell.$ There are some important complications. The magnitude
of $Y_\ell$ may be too large to apply the smoothed analysis argument directly
to~$Y_\ell.$ Instead we observe that the columns of $X_\ell$ are contained in
the range~$S$ of the $n\times 2k$ matrix $[U \mid ( NX_{\ell-1} + G_\ell + H_\ell)].$ This
is because $Y_\ell = AX_{\ell-1} + G_\ell + H_\ell$ and
$AX_{\ell-1}=MX_{\ell-1}+NX_{\ell-1}$ where $M=U\Lambda U^\trans$ and
$N=(I-UU^\trans)A.$ Since~$S$ has
dimension at most $2k$ it suffices to argue that this space has small
coherence. Moreover we can choose $H_\ell$ to be roughly on the same order as
$NX_{\ell-1}$ and $G_\ell$ so that the smoothed analysis argument leads to an
excellent bound bound on the smallest singular value of $NX_{\ell-1} + G_\ell +
H_\ell.$ To prove that the coherence is small we need to exhibit a basis
for~$S.$ This requires us to argue about the related matrix
$(I-UU^\trans)(NX_{\ell-1}+G+H_\ell)$ since we need to orthonormalize the last
$k$ vectors against the first when constructing a basis. 
Another minor complication is that we don't know the magnitude of $G_\ell$ so we
need to find the right scaling of $H_\ell$ on the fly.  We call the resulting
procedure that \SmoothGS and analyze its guarantees in \sectionref{smoothgs}. 

\paragraph{Putting things together.}
The final algorithm that we analyze is quite simple to describe as shown
in~\figureref{saltls}. The algorithm makes use of an initialization procedure
\Init that we defer to~\sectionref{init}. In \sectionref{convergence} we prove
our main theorem. At a high-level, the theorem is proved by induction. The
main inductive hypothesis is that the coherence of the $\ell$-th solution
$X_\ell$ is small, i.e., bounded in terms of the coherence parameter $\mustar.$ 
Given that the the coherence is small we can control the magnitude of the
noise term $G_{\ell+1}$ using matrix concentration inequalities. Given that
$G_{\ell+1}$ is small in spectral norm, our results on the noisy power method
show that the algorithm makes progress towards convergence.
To ensure that the inductive hypothesis
continues to hold we use our analysis of the smooth orthonormalization. 

The generalization of our result to rectangular matrices
follows from a standard ``dilation'' argument that we describe in
\sectionref{rectangular}.

The description of the algorithm also uses a helper function called \Split
that's used to split the subsample into independent pieces of roughly equal
size while preserving the distributional assumption that our theorems use. We
discuss \Split in \sectionref{split}.

\begin{figure}[ht]
\begin{boxedminipage}{\textwidth}
\noindent \textbf{Input:} 
Observed set of indices
$\Omega\subseteq [n]\times[n]$ of an unknown symmetric matrix $A\in\R^{n\times
n}$ with entries $P_\Omega(A),$ number of
iterations $L\in\mathbb{N},$ error parameter~$\epsilon>0,$ target dimension $k,$ 
coherence parameter~$\mu.$ 

\noindent \textbf{Algorithm} $\text{\SAltLS}(P_\Omega(A),\Omega,L,k,\epsilon,\mu):$
\begin{enum}
\item $(\Omega_0,\Omega')\leftarrow{\text{\Split}}(\Omega,2),$
$(\Omega_1,\dots,\Omega_L)\leftarrow{\text{\Split}}(\Omega',L)$
\item $X_0 \leftarrow \text{\Init}(P_{\Omega_0}(A),\Omega_0,k,\mu)$
\item For $\ell = 1$ to $L$:
\begin{enum}
\item 
$Y_\ell \leftarrow
\text{\MedianLS}(P_{\Omega_\ell}(A),\Omega_\ell,X_{\ell-1},L,k)$
\item $X_\ell \leftarrow \text{\SmoothQR}(Y_\ell,\epsilon,\mu)$
\end{enum}
\end{enum}
\noindent \textbf{Output:} Pair of matrices $(X_{L-1},Y_L)$ 
\end{boxedminipage}
\caption{Smoothed alternating least squares (\SAltLS)}
\figurelabel{saltls}
\end{figure}

\subsection{Further discussion of related work}

There is a vast literature on the topic that we cannot completely survey here.
Most closely related is the work of Jain et al.~\cite{JainNS13} that suggested
the idea of thinking of alternating least squares as a noisy update step in
the Power Method.  Our approach takes inspiration from 
this work by analyzing least squares using the noisy power method. However,
our analysis is substantially different in both how convergence and low
coherence is argued. The approach of Keshavan~\cite{Keshavan12} uses a rather
different argument. 

As an alternative to the nuclear norm approach, Keshavan, Montanari and
Oh~\cite{KeshavanMO10,KeshavanMO10b} present two approaches, a spectral
approach and an algorithm called {\sc OptSpace}. The spectral approach roughly
corresponds to our initialization procedure and gives similar guarantees. 
{\sc OptSpace} requires a stronger incoherence assumption, has larger sample
complexity in terms of the condition number, namely $(\sigma_1/\sigma_k)^6,$
and requires optimizing over the Grassmanian manifold. However, the
requirement on~$N$ achieved by {\sc OptSpace} can be weaker than ours in the
noisy setting. In the exact case, our algorithm has a much faster convergence
rate (logarithmic dependence on $1/\eps$ rather than polynomial). 

There are a number of fast algorithms for matrix completion based on either
(stochastic) gradient descent~\cite{RechtR13} or (online)
Frank-Wolfe~\cite{JaggiS10,HazanK12}.  These algorithms generally minimize
squared loss on the \emph{observed} entries subject to a nuclear norm
constraint and in general do not produce a matrix that is close to the true
unknown matrix on all entries. In contrast, our algorithm guarantees
convergence \emph{in domain}, that is, to the unknown matrix itself. Moreover,
our dependence on the error is logarithmic whereas in these algorithms it is
polynomial.

\paragraph{Privacy-preserving spectral analysis.} 
Our work is also closely related to a line of work on differentially private
singular vector computation~\cite{HardtR12,HardtR13,Hardt13}. These papers
each consider algorithms based on the power method where noise is injected to
achieve the privacy guarantee known as Differential Privacy~\cite{DworkMNS06}. 
Hardt and Roth~\cite{HardtR12,HardtR13,Hardt13} observed that incoherence could
be used to obtain improved guarantees. This requires controlling the
coherence of the iterates produced by the noisy power method which leads to
similar problems as the ones faced here. What's simpler in the privacy setting
is that the noise term is typically Gaussian leading to a cleaner analysis. 
Our work uses a similar convergence analysis for noisy subspace iteration that 
was used in a concurrent work by the author~\cite{HardtR13}. 

\subsection{Preliminaries and Notation}
\sectionlabel{prelims}
We denote by $A^\trans$ the transpose of a matrix (or vector) $A.$ 
We use the
notation $x \gtrsim y$ do denote that the relation $x \ge C y$ holds for a
sufficiently large absolute constant $C>0$ independent of $x$ and~$y.$ We let
$\cR(A)$ denote the range of the matrix~$A.$
\begin{definition}[Coherence]
\definitionlabel{coherence}
The $\mu$-\emph{coherence} of a $k$-dimensional subspace $U$ of $\R^n$ is
defined as
$\mu(U) \defeq \max_{i\in[n]}\frac nk\Norm{P_Ue_i}_2^2\mcom$
where $e_i$ denotes the $i$-th standard basis vector.
\end{definition}

\section{Robust local convergence of subspace iteration}
\sectionlabel{nsi}

\figureref{subspaceit} presents our basic template algorithm. The algorithm is
identical to the standard subspace iteration algorithm except that
in each iteration~$\ell$, the computation is perturbed by a matrix $G_\ell.$
The matrix $G_\ell$ can be adversarially and adaptively chosen in each round.
We will analyze under which conditions on the perturbation we can expect the
algorithm to converge rapidly.

\begin{figure}[ht]
\begin{boxedminipage}{\textwidth}
\noindent \textbf{Input:} Matrix $A\in\mathbb{R}^{n\times n},$ number of
iterations $L\in\mathbb{N},$ target dimension $k$
\begin{enum}
\item Let $X_0\in\R^{n\times k}$ be an orthonormal matrix.
\item For $\ell = 1$ to $L$:
\begin{enum}
\item Let $G_\ell\in\R^{n\times k}$ be an arbitrary perturbation.
\item\itemlabel{mult} $Y_\ell \leftarrow AX_{\ell-1}+G_\ell$
\item $X_\ell \leftarrow \mathrm{GS}(Y_\ell)$
\end{enum}
\end{enum}
\noindent \textbf{Output:} Matrix $X_L$ with $k$ orthonormal columns
\end{boxedminipage}
\caption{Noisy Subspace Iteration (\NSI)}
\figurelabel{subspaceit}
\end{figure}

Principal angles are a useful tool in analyzing the convergence behavior of
numerical eigenvalue methods. We will use the largest principal angle between
two subspaces as a potential function in our convergence analysis.
\begin{definition}
\definitionlabel{angles}
Let $X,Y\in\mathbb{R}^{n\times k}$ be orthonormal bases for subspaces
$\cX,\cY,$ respectively. Then, the sine of the \emph{largest principal angle}
between $\cX$ and $\cY$ is defined as 
$\sin\theta(\cX,\cY)\defeq \Norm{(I-XX^\trans)Y}\mper$
\end{definition}

We use some standard properties of the largest principal angle. 

\begin{proposition}[\cite{ZhuK12}]
\propositionlabel{identities}
Let $\cX,\cY,X,Y$ be as in \definitionref{angles} and
let $X_\bot$ be  an orthonormal basis for the orthogonal complement of~$\cX.$
Then, we have $\cos\theta(\cX,\cY)= \sigma_k(X^\trans Y).$
and assuming $X^\trans Y$ is invertible, 
$\tan\theta(\cX,\cY)=\|X_\bot^\trans Y (X^\trans Y)^{-1}\|$
\end{proposition}

From here on we will always assume that $A$ has the spectral
decomposition
\begin{equation}\equationlabel{decomp}
A = U\Lambda_U U^\trans + V\Lambda_V V^\trans\mcom
\end{equation}
where $U\in\R^{n\times k},V\in\R^{n\times (n-k)}$ corresponding to the first
$k$ and last $n-k$ eigenvectors respectively. We will let $\sigma_1\ge\dots\ge
\sigma_n$ denote the singular values of $A$ which coincide with the absolute 
eigenvalues of $A$ sorted in non-increasing order.

Our convergence analysis tracks the tangent of the largest principal angles between the subspaces $\cR(U)$ and $\cR(X_\ell)$.  The next lemma shows a natural condition under which 
the potential decreases multiplicatively in step~$\ell.$ We think of this
lemma as a local convergence guarantee, since it assumes that the cosine of
the largest principal angle between $\cR(U)$ and $\cR(X_{\ell-1})$ is already lower
bounded by a constant.

\begin{lemma}[One Step Local Convergence]\lemmalabel{one}
Let $\ell\in\Set{1,\dots,L}.$ Assume that 
\[
\cos\theta_k(U,X_{\ell-1})\ge \frac12 > \frac{\|U^\trans G_\ell\|}{\sigma_k}\mper
\] 
Then,
\begin{equation}\equationlabel{tandrop}
\tan\theta(U,X_\ell)
\le 
\tan\theta(U,X_{\ell-1})
\cdot\frac{\sigma_{k+1}+\frac{2\|V^\trans G_\ell\|}{\tan\theta(U,X_{\ell-1})}}
{\sigma_k-2\|U^\trans G_\ell\|}\mper
\end{equation}
\end{lemma}

\begin{proof}
We first need to verify that $X_\ell$ has rank~$k.$ This follows if we can
show that $\sigma_k(Y_\ell)>0.$  Indeed,
\[
\sigma_k(Y_\ell)\ge\sigma_k(U^\trans Y_\ell)
= \sigma_k(\Lambda_U U^\trans X_{\ell-1}+U^\trans G_\ell)
\ge \sigma_k \cdot \sigma_k(U^\trans X_{\ell-1}) - \|U^\trans G_\ell\|\mper
\]
The right hand side is strictly greater than zero by our assumption, because
$\sigma_k(U^\trans X_{\ell-1})=\cos\theta_k(U,X_{\ell-1}).$
Further, we have $X_\ell=Y_\ell R$ for some invertible
transformation~$R.$ Therefore, $U^\trans X_\ell$ is invertible and we can
invoke \propositionref{identities} to express $\tan\theta(U,X_\ell)$ as:
\[
\Norm{V^\trans X_\ell(U^\trans X_\ell)^{-1}}
= \Norm{V^\trans Y_\ell R R^{-1}(U^\trans Y_\ell)^{-1}}
= \Norm{V^\trans Y_\ell (U^\trans Y_\ell)^{-1}}\mper
\]
Using the fact that $Y_\ell=AX_{\ell-1}+G_\ell,$
\begin{align*}
\Norm{V^\trans Y_\ell (U^\trans Y_\ell)^{-1}}
&= \Norm{V^\trans Y_\ell (\Lambda_U U^\trans X_{\ell-1}+ U^\trans G_\ell)^{-1}}\\
&= \Norm{V^\trans Y_\ell \left(\big(\Lambda_U + U^\trans G_\ell (U^\trans
X_{\ell-1})^{-1}\big) U^\trans X_{\ell-1}\right)^{-1}}\\
&= \Norm{V^\trans Y_\ell (U^\trans X_{\ell-1})^{-1}
\left(\Lambda_U + U^\trans G_\ell (U^\trans X_{\ell-1})^{-1}\right)^{-1}}
\end{align*}
Putting $S=\Lambda_U + U^\trans G_\ell (U^\trans X_{\ell-1})^{-1},$ we
therefore get
\begin{align*}
\Norm{V^\trans Y_\ell (U^\trans Y_\ell)^{-1}}
&\le \Norm{V^\trans Y_\ell(U^\trans X_{\ell-1})^{-1} S^{-1}}\\
&\le \Norm{V^\trans Y_\ell(U^\trans X_{\ell-1})^{-1}}\cdot\Norm{S^{-1}}
 = \frac{\Norm{V^\trans Y_\ell(U^\trans X_{\ell-1})^{-1}}}{\sigma_k(S)}\mper
\end{align*}
In the second inequality we used the fact that for any two matrices $P,Q$ we
have $\|PQ\|\le\|P\|\cdot\|Q\|.$
Let us bound the numerator of the RHS as follows:
\begin{align*}
\Norm{V^\trans Y_\ell(U^\trans X_{\ell-1})^{-1}}
& = \Norm{\Lambda_V V^\trans X_{\ell-1}(U^\trans X_{\ell-1})^{-1}+ V^\trans
G_\ell (U^\trans X_{\ell-1})^{-1}} \\
& = \Norm{\Lambda_VV^\trans X_{\ell-1}(U^\trans
X_{\ell-1})^{-1}} + \Norm{V^\trans G_\ell (U^\trans X_{\ell-1})^{-1}} \\
& = \Norm{\Lambda_V}\cdot \Norm{V^\trans X_{\ell-1}(U^\trans
X_{\ell-1})^{-1}} + \Norm{V^\trans G_\ell (U^\trans X_{\ell-1})^{-1}} \\
& = \sigma_{k+1}\cdot \tan\theta(U,X_{\ell-1})
+ \Norm{V^\trans G_\ell (U^\trans X_{\ell-1})^{-1}}\\
& \le \sigma_{k+1}\cdot \tan\theta(U,X_{\ell-1})
+ \Norm{V^\trans G_\ell}\cdot\Norm{(U^\trans X_{\ell-1})^{-1}}\\
& = \sigma_{k+1}\cdot \tan\theta(U,X_{\ell-1})
+ \frac{\Norm{V^\trans G_\ell}}{\cos\theta_k(U,X_{\ell-1})}\\
& = \sigma_{k+1}\cdot \tan\theta(U,X_{\ell-1})
+ 2\Norm{V^\trans G_\ell}
\mper
\end{align*}
Here we used the fact that 
\[
\|(U^\trans X_{\ell-1})^{-1}\|=\frac1{\sigma_k(U^\trans X_{\ell-1})}=
\frac1{\cos\theta_k(U,X_{\ell-1})}\mper
\]
We also need a lower bound on $\sigma_k(S).$ Indeed,
\begin{align*}
\sigma_k(S) 
& \ge \sigma_k(\Lambda_U) - \Norm{U^\trans G_\ell(U^\trans X_{\ell-1})^{-1}}\\
& \ge \sigma_k - \Norm{U^\trans G_\ell}\cdot\Norm{(U^\trans X_{\ell-1})^{-1}}
= \sigma_k - 2\Norm{U^\trans G_\ell}\mper
\end{align*}
Note that the RHS is strictly  positive due to the assumption of the
lemma.  Summarizing what we have,
\[
\tan\theta(U,X_\ell) \le
\frac{\sigma_{k+1}\cdot \tan\theta(U,X_{\ell-1})
+ 2\Norm{V^\trans G_\ell}}
{\sigma_k - 2\Norm{U^\trans G_\ell}}\mper
\]
This is equivalent to the statement of the lemma as we can see from a simple
rearrangement.
\end{proof}

The next lemma essentially follows by iterating the previous lemma.

\begin{lemma}[Local Convergence]
\lemmalabel{convergence}
Let $0\le \epsilon\le1/4.$
Let $\Delta = \max_{1\le \ell \le L}\|G_\ell\|$ and $\gamma_k=
1-\sigma_{k+1}/\sigma_k.$
Assume that $\|V^\trans X_0\|\le 1/4$
and $\sigma_k \ge 8\Delta/\gamma_k\epsilon\mper$
Then,
\[
\Norm{V^\trans X_L} \le
\max\left\{\epsilon,2\cdot \Norm{V^\trans X_0}\cdot\exp(-\gamma_k L/2)\right\}\mper
\]
\end{lemma}

\begin{proof}
Our first claim shows that once the potential function is below $\epsilon$ at step
$\ell-1$, it cannot increase beyond $\epsilon.$
\begin{claim}
Let $\ell\ge1.$
Suppose that $\tan\theta(U,X_{\ell-1})\le \epsilon.$ 
Then, $\tan\theta(U,X_{\ell})\le \epsilon.$
\end{claim}
\begin{proof}
By our assumption, $\cos\theta_k(U,X_{\ell-1})\ge \sqrt{1-\epsilon^2}\ge
15/16.$ Together with the lower
bound on $\sigma_k,$ the assumptions for \lemmaref{one} are met. Hence,
using our assumptions,
\[
\tan\theta(U,X_\ell) \le \frac{(1-\gamma_k)\sigma_{k}\epsilon+2\Delta}{\sigma_k-2\Delta}
\le \epsilon\mper\qedhere
\]
\end{proof}

Our second claim shows that if the potential is at least $\epsilon$
at step $\ell-1,$ it will decrease by a factor $1-\gamma_k/2.$

\begin{claim}
Let $\ell\ge 1$ Suppose that
$\tan\theta(U,X_{\ell-1})\in[\epsilon,1/2].$ Then,
\[
\tan\theta(U,X_{\ell})\le(1-\gamma_k/2)\tan\theta(U,X_{\ell-1})\mper
\]
\end{claim}
\begin{proof}
Using the assumption of the claim
we have $\cos\theta(U,X_{\ell-1})\ge
\frac1{\tan\theta(U,X_{\ell-1})}\ge 1/2>\Delta/\sigma_k.$ 
We can therefore apply \lemmaref{one} to conclude
\begin{align*}
\tan\theta(U,X_{\ell})
&\le \tan\theta(U,X_{\ell-1})
\cdot\frac{(1-\gamma_k)\sigma_k+2\Delta}{\sigma_k-2\Delta}\\
&\le \tan\theta(U,X_{\ell-1})
\cdot\frac{(1-\gamma_k)(1+\gamma_k/4)}{1-\gamma_k/4}
 \le \tan\theta(U,X_{\ell-1})(1-\gamma_k/2)\qedhere
\end{align*}
\end{proof}
The two previous claims together imply that
\[
\tan\theta(U,X_L)\le
\max\Set{\tan\theta(U,X_0)(1-\gamma_k/2)^L,\epsilon}\mcom
\]
provided that $\tan\theta(U,X_0)\le1/2.$ This is the case since  we
assumed that $\sin\theta(U,X_0)\le1/4.$
Note that $(1-\gamma_k/2)^L\le \exp(-\gamma_k L/2).$
It remains to observe that 
$\Norm{V^\trans X_L} \le \tan\theta(U,X_L)$
and further $\tan\theta(U,X_0)\le 2\Norm{V^\trans X_0}$ by our assumption 
on~$X_0.$
\end{proof}

In our application later on the error terms $\|G_\ell\|$ decrease as $\ell$
increases and the algorithm starts to converge. We need a convergence bound
for this type of shrinking error. The next definition expresses a condition on $G_\ell$
that allows for a useful convergence bound.

\begin{definition}[Admissible]
\definitionlabel{admissible}
Let $\gamma_k=1-\sigma_{k+1}/\sigma_k.$
We say that the pair of matrices $(X_{\ell-1},G_\ell)$ is
\emph{$\epsilon$-admissible for \NSI} if 
\begin{equation}\equationlabel{admissble}
\|G_\ell\|\le \frac1{32}\gamma_k\sigma_k\|V^\trans X_{\ell-1}\|
+ \frac{\epsilon}{32}\gamma_k\sigma_k.
\end{equation}
We say that a family of matrices $\Set{(X_{\ell-1},G_\ell)}_{\ell=1}^L$ is
\emph{$\epsilon$-admissible for \NSI} if each member of the set is $\epsilon$-admissible. We
will use the notation $\{G_\ell\}$ as a shorthand for 
$\Set{(X_{\ell-1},G_\ell)}_{\ell=1}^L.$
\end{definition}

We have the following convergence guarantee for admissible noise matrices.

\begin{theorem}\theoremlabel{convergence}
Let $\gamma_k=1-\sigma_{k+1}/\sigma_k.$ Let $\epsilon\le1/2.$
Assume that the family of noise matrices $\{G_\ell\}$ is
$(\epsilon/2)$-admissible for \NSI  and that $\Norm{V^\trans X_0}\le 1/4.$
Then, we have $\Norm{V^\trans X_L}\le \epsilon$
for any $L\ge 4\gamma_k^{-1}\log(1/\epsilon).$ 
\end{theorem}
\begin{proof}
We prove by induction that for every $t\ge0$ 
after $L_t = 4t\gamma_k^{-1}$ steps, we have
\[
\Norm{V^\trans X_{L_t}}\le \max\Set{2^{-(t+1)},\epsilon}\mper
\] 
The base case ($t=0$) follows directly from 
the assumption that $\Norm{V^\trans X_0}\le 1/4.$
We turn to the inductive step.
By induction hypothesis, we have $\Norm{V^\trans X_{L_{t}}}\le
\max\Set{2^{-(t+1)},\epsilon}.$
We apply \lemmaref{convergence} with ``$X_0=X_{L_t}$'' and
error parameter $\max\Set{2^{-{t+2}},\epsilon}$ and $L=L_{t+1}-L_t.$ 
The conditions of the lemma are satisfied as can be easily checked using the
assumption that $\{G_\ell\}$ is $\epsilon/2$-admissible. 
Using the fact that $L_{t+1}-L_t=4/\gamma_k,$ the conclusion of the lemma gives
\[
\Norm{V^\trans X_{L_{t+1}}}
\le\max\Set{\epsilon,2\cdot\max\Set{\epsilon,2^{-(t+1)}}\exp\left(-\frac{\gamma_k
(L_{t+1}-L_t)}2\right)}
\le\max\Set{\epsilon,2^{-(t+2)}}\mper\qedhere
\]
\end{proof}

\section{Least squares update rule}
\sectionlabel{altls}

\begin{figure}[ht]
\begin{boxedminipage}{\textwidth}
\noindent \textbf{Input:} 
Target dimension $k,$ observed set of indices
$\Omega\subseteq [n]\times[n]$ of an unknown symmetric matrix $A\in\R^{n\times
n}$ with entries $P_\Omega(A),$ orthonormal matrix $X\in\R^{n\times k}.$

\noindent \textbf{Algorithm} $\text{\AltLS}(P_\Omega(A),\Omega,X,L,k):$
\begin{enumerate}
\item[~] $Y\leftarrow \arg\min_{Y\in\R^{n\times k}}\Norm{P_{\Omega}(A-XY^\trans)}_F^2$
\end{enumerate}
\noindent \textbf{Output:} Pair of matrices $(X,Y)$ 
\end{boxedminipage}
\caption{Least squares update}
\figurelabel{altmin}
\end{figure}

\figureref{altmin} describes the least squares update step
specialized to the case of a symmetric matrix. 
Our goal is to express this update step as an update step of the form $Y=AX+G$
so that we may apply our analysis of noisy subspace iteration. This syntactic
transformation is explained in \sectionref{equiv} followed by 
a bound on the norm of the error term $G$ in \sectionref{large-dev}.

\subsection{From alternating least squares to noisy subspace iteration}
\sectionlabel{equiv}

The optimizer $Y$ satisfies a set of linear equations that we derive from
the gradient of the objective function.
\begin{lemma}[Optimality Condition]
\lemmalabel{optimality}
Let $P_i\colon\R^n\to\R^n$ be the linear
projection onto the coordinates in $\Omega_{i}=\{j\colon (i,j)\in\Omega\}$ scaled by
$p^{-1} = n^2/(\E|\Omega|),$ i.e., 
$P_i = p^{-1}\sum_{j\in\Omega_{i}} e_je_j^\trans\mper$
Further, define the matrix $B_i\in\R^{k\times k}$ as
$B_i = X^\trans P_i X$ and assume that $B_i$ is invertible. Then, for every
$i\in[n],$ the $i$-th row of $Y$ satisfies 
$e_i^\trans Y = e_i^\trans A P_i X B_i^{-1}\mper$
\end{lemma}
\begin{proof}
Call the objective function $f(Y)=\|P_{\Omega}(A-XY^\trans)\|_F^2\mper$
We note that for every
$i\in[n],j\in[k],$ we have
$\frac{\partial f}{\partial Y_{ij}} 
 =  -2\sum_{s\in\Omega_i} A_{is}X_{sj}
+ 2 \sum_{r=1}^k Y_{ir}\sum_{s\in \Omega_i} X_{sj}X_{sr} \mper$
From this we conclude that the optimal $Y$ must satisfy 
$e_i^\trans A P_i X =e_i^\trans Y X^\trans P_i X =e_i^\trans Y B_i.$
Hence,
$e_i^\trans Y = e_i^\trans AP_iX B_i^{-1}\mper$
\end{proof}

The assumption that $B_i$ is invertible is essentially without loss of
generality. Indeed, we will later see that $B_i$ is invertible (and in fact close 
to the identity matrix) with very high probability. 
We can now express the least squares update as $Y = AX+G$ 
where we derive some useful expression for~$G.$
\begin{lemma}
\lemmalabel{error-term}
Let $E = (I-XX^\trans)U.$ We have $Y = AX + G$ 
where $G = G^M + G^N$ and the matrices
$G^M$ and $G^N$ satisfy for each row $i\in[n]$ if $B_i$ is invertible then
\begin{align*}
e_i^\trans G^M 
& = e_i^\trans U\Lambda_U E^\trans P_i X B_i^{-1}\\
e_i^\trans G^N 
& = e_i^\trans ( NP_i X B_i^{-1} - NX)
\mper
\end{align*}
\end{lemma}

\begin{proof}
By \lemmaref{optimality},
$e_i^\trans Y = 
e_i^\trans A P_i X B_i^{-1}
= e_i^\trans Y = e_i^\trans M P_i X B_i^{-1} + e_i^\trans N P_i X
B_i^{-1}\mper$
Let $C_i = U^\trans P_i X$
and put $D = U^\trans X.$ On the one hand,
\begin{align*}
e_i^\trans M P_i XB_i^{-1}
 = e_i^\trans U \Lambda_U  C_i B_i^{-1} 
& = e_i^\trans ( U \Lambda_U D - U\Lambda_U (DB_i- C_i) B_i^{-1}) \\
& = e_i^\trans MX
- e_i^\trans U\Lambda_U (DB_i- C_i) B_i^{-1}
\end{align*}
On the other hand,
\begin{align*}
C_i  = U^\trans P_i X
 = (XX^\trans U + E)^\trans P_i X
 = (U^\trans X)X^\trans P_i X + E^\trans (P_i X)
= DB_i + E^\trans P_i X\mper
\end{align*}
Hence, as desired,
$e_i^\trans MP_i XB_i^{-1} = 
e_i^\trans MX - e_i^\trans U\Lambda_U E^\trans P_i X B_i^{-1} \mper$
Finally, it follows directly by definition that
$e_i^\trans N P_i X B_i^{-1}
= e_i^\trans NX + e_i^\trans G^N.$
Putting the previous two equations together, we conclude that
$Y = 
M X + G^M 
+ N X + G^N
=
AX + G^M + G^N\mper$
\end{proof}

\subsection{Deviation bounds for the least squares update}
\sectionlabel{large-dev}

In this section we analyze the norm of the error term $G$ from the previous
section. More specifically, we prove a bound on the norm of
each row of $G.$ Our bound uses the fact that the matrix~$E$ appearing in the
expression for the error term satisfies $\Norm{E}=\Norm{V^\trans X}.$ This
gives us a bound in terms of the quantity $\Norm{V^\trans X}.$ 
\begin{lemma}
\lemmalabel{altls-error}
Let $\delta\in(0,1).$ 
Assume that each entry is included in~$\Omega$ independently with probability
\begin{equation}
\equationlabel{LS-sample}
p \gtrsim  \frac{k\mu(X)\log n}{\delta^2 n}\mper
\end{equation}
Then, for every $i\in[n],$
$\Pr\Set{\Norm{e_i^\trans G}> \delta\cdot\left(\|e_i^\trans M\|\cdot\Norm{V^\trans X}
+ \|e_i^\trans N\|\right)}\le \frac1{5}\mper$
\end{lemma}
\begin{proof}
Fix $i\in[n]$ and write $a_i^\trans=e_i^\trans A$ and
$z_i^\trans=a_i^\trans(I-XX^\trans)$.  We condition on the history generating
$X$.  The current sample $\Omega$ is fresh and independent of this history;
since $A$ is fixed, conditional on the history both $X$ and $z_i$ are fixed,
while $P_i$ has the distribution of $P$ in
Lemma~\ref*{lem:row-sampling}.  By the optimality condition and the
identity $B_i=X^\trans P_iX$,
\begin{align*}
e_i^\trans G
&=a_i^\trans P_iXB_i^{-1}-a_i^\trans X\\
&=(a_i^\trans P_iX-a_i^\trans XB_i)B_i^{-1}\\
&=z_i^\trans P_iXB_i^{-1}.
\end{align*}
Moreover, $z_i^\trans X=0$, and, since
$E=(I-XX^\trans)U$ and $\Norm{E}=\Norm{V^\trans X}$,
\begin{align*}
\Norm{z_i}
&\le \Norm{e_i^\trans M(I-XX^\trans)}
   +\Norm{e_i^\trans N(I-XX^\trans)}\\
&\le \Norm{e_i^\trans U\Lambda_U}\Norm{E}+\Norm{e_i^\trans N}\\
&=\Norm{e_i^\trans M}\Norm{V^\trans X}+\Norm{e_i^\trans N}.
\end{align*}
Apply Lemma~\ref*{lem:row-sampling} with error parameter $\delta/2$ and
Lemma~\ref*{lem:inverse-bound} with error parameter $1/2$.  The assumed lower
bound on $p$ implies the hypotheses of both lemmas.  Except on an event of
probability at most $1/10+n^{-5}\le 1/5$, we have simultaneously
\[
\Norm{z_i^\trans P_iX}\le \frac{\delta}{2}\Norm{z_i}
\qquad\text{and}\qquad
\Norm{B_i^{-1}}\le 2.
\]
Consequently,
\[
\Norm{e_i^\trans G}
\le \delta\Norm{z_i}
\le \delta\left(\Norm{e_i^\trans M}\Norm{V^\trans X}
                 +\Norm{e_i^\trans N}\right)\mper\qedhere
\]
\end{proof}

\subsection{Median least squares update}

Given the previous error bound we can achieve a strong concentration bound 
by taking the component-wise median of multiple independent samples of the error term.
\begin{lemma}\lemmalabel{median-lemma}
\lemmalabel{medianls-error}
Let $G_1,\dots,G_t$ be i.i.d.~copies of $G.$ 
Let $\bar G = \mathrm{median}(G_1,\dots,G_t)$ be the component-wise median of
$G_1,\dots,G_t$ and assume $p$ satisfies \eqref{LS-sample}. Then, for
every $i\in[n],$
\[
\Pr\Set{\Norm{e_i^\trans \bar G}> \delta\left(\|e_i^\trans M\|\cdot\Norm{V^\trans X}
+ \|e_i^\trans N\|\right)}\le \exp(-\Omega(t))\mper
\]
\end{lemma}
\begin{proof}
Fix $i\in[n]$ and let $g_1,\dots,g_t\in\R^k$ denote the $i$-th rows of
$G_1,\dots,G_t.$ Let $S = \{ j\in[t]\colon \|g_j\|\le B\}$ where $B = 
(\delta/4)\left(\|e_i^\trans M\|\cdot\Norm{V^\trans X} + \|e_i^\trans N\|\right).$
Applying \lemmaref{altls-error} with error parameter $\delta/4$ it follows
that $\E|S|\ge 4t/5.$ Moreover, the draws of $g_j$ are independent. So we can
apply a Chernoff bound to argue that $|S|>2t/3$ with probability
$1-\exp(-\Omega(t)).$ Assuming that this event occurs, we claim that
$\bar g = \mathrm{median}(g_1,\dots,g_t)$ satisfies $\Norm{\bar g}\le 4B$ and
this claim establishes the lemma. 

To prove this claim, fix any coordinate of $r\in[k].$ 
By the median property,
$|\{j\colon (g_j)_r^2\ge \bar g_r^2\}|\ge t/2$.  Since $|S|>2t/3$,
the intersection of these two sets has size greater than $t/6$.  Therefore,
the average of $(g_j)_r^2$ over $j\in S$ is at least
\[
\frac{t}{6|S|}\bar g_r^2\ge \frac16\bar g_r^2.
\]
Summing over $r$ shows that the average of $\Norm{g_j}^2$ over $j\in S$ is
at least $\Norm{\bar g}^2/6$.  By the definition of $S$, this average is at
most $B^2$.  Hence $\Norm{\bar g}^2\le 6B^2$, and in particular
$\Norm{\bar g}\le \sqrt6B<4B$, as required.
\end{proof}

We can now conclude a strong concentration bound for the median of multiple
independent solutions to the least squares minimization step. This way we can
obtain the desired error bound for all rows simultaneously. This leads to the
following extension of the least squares update rule.

\begin{figure}[ht]
\begin{boxedminipage}{\textwidth}
\noindent \textbf{Input:} 
Target dimension $k,$ observed set of indices
$\Omega\subseteq [n]\times[n]$ of an unknown symmetric matrix $A\in\R^{n\times
n}$ with entries $P_\Omega(A),$ orthonormal matrix $X\in\R^{n\times k}.$

\noindent \textbf{Algorithm} $\text{\MedianLS}(P_\Omega(A),\Omega,X,L,k):$
\begin{enum}
\item $(\Omega_1,\dots,\Omega_t)\leftarrow\text{\Split}(\Omega,t)$ for
$t=O(\log n).$
\item $Y_i\leftarrow \text{\AltLS}(P_{\Omega_i}(A),\Omega_i,X,L,k)$
\end{enum}
\noindent \textbf{Output:} Pair of matrices $(X,\mathrm{median}(Y_1,\dots,Y_t))$ 
\end{boxedminipage}
\caption{Median least squares update}
\figurelabel{medianls}
\end{figure}

\begin{lemma}\lemmalabel{median-bound}
Let $\Omega$ be a sample in which each entry is 
included independently with probability 
$p \gtrsim  \frac{k\mu(X)\log^2 n}{\delta^2 n}\mper$
Let $Y \leftarrow \text{\MedianLS}(P_{\Omega}(A),\Omega,X,L,k).$ 
Then, we have with probability $1-1/n^3$ that 
$\bar Y=AX+\bar G$ and $\bar G$ satisfies for every $i\in[n]$ the bound
$\Norm{e_i^\trans \bar G}\le \delta \Norm{e_i^\trans M}\cdot\Norm{V^\trans
X}+\delta\Norm{e_i^\trans N}\mper$
\end{lemma}
\begin{proof}
By \lemmaref{split}, the samples $\Omega_1,\dots,\Omega_t$ are independent and
each set $\Omega_j$ includes each entry with probability at least $p/t.$
The output satisfies $Y=\mathrm{median}(Y_1,\dots,Y_t),$
where each $Y_j$ is of the form $Y_j=AX+G_j.$ 
It follows that
$\mathrm{median}(Y_1,\dots,Y_t)=AX + \bar G$ where $\bar
G=\mathrm{median}(G_1,\dots,G_t).$ We can therefore apply
\lemmaref{median-lemma} to conclude the lemma using the fact that $t=O(\log
n)$ allows us to take a union bound over all~$n$ rows.
\end{proof}

\section{Incoherence via smooth QR factorization}
\sectionlabel{smoothgs}

As part of our analysis of alternating minimization we need to show that the
intermediate solutions $X_\ell$ have small coherence. For this purpose we
propose an idea inspired by Smoothed Analysis of the QR
factorization~\cite{SankarST06}. The problem with applying the QR
factorization directly to $Y_\ell$ is that $Y_\ell$ might be ill-conditioned.
This can lead to a matrix $X_\ell$ (via QR-factorization) that has large
coordinates and whose coherence is therefore no longer as small as we desire.
A naive bound on the condition number of~$Y_\ell$ would lead to a large loss
in sample complexity.
What we show instead is that a small Gaussian perturbation to $Y_\ell$ leads
to a sufficiently well-conditioned matrix $\tilde Y_\ell = Y_\ell + H_\ell.$ 
Orthonormalizing $\tilde Y_\ell$ now leads to a matrix of small coherence.
Intuitively, since the computation of $Y_\ell$ is already noisy the additional
noise term has little effect so long as its norm is bounded by that
of $G_\ell.$ Since we don't know the norm of $G_\ell,$ we have to
search for the right noise parameter using a simple binary search. We call the
resulting procedure \SmoothQR and describe in in \figureref{SmoothGS}.

\begin{figure}[ht]
\begin{boxedminipage}{\textwidth}
\noindent \textbf{Input:} 
Matrix $Y\in\R^{n\times k},$ parameters $\mu,\epsilon >0.$

\noindent \textbf{Algorithm} $\text{\SmoothGS}(Y,\epsilon,\mu):$
\begin{enum}
\item $X\leftarrow \text{\GS}(Y), H\leftarrow 0, \sigma \leftarrow \epsilon\|Y\|/n.$
\item While $\mu(X) > \mu$ and $\sigma\le\|Y\|$:
\begin{enum}
\item $X \leftarrow \mathrm{GS}(Y + H)$ where $H\sim
\mathrm{N}(0,\sigma^2/n)^{n\times k}$
\item $\sigma \leftarrow 2\sigma$
\end{enum}
\end{enum}
\noindent \textbf{Output:} Pair of matrices $(X,H)$ 
\end{boxedminipage}
\caption{Smooth Orthonormalization (\SmoothGS)}
\figurelabel{SmoothGS}
\end{figure}

To analyze the algorithm we begin with a lemma that analyzes the smallest
singular value under a Gaussian perturbation.  What makes the analysis easier
is the fact that the matrices we're interested in are rectangular. The square
case was considered in~\cite{SankarST06} and requires more involved arguments.

\begin{lemma}
\lemmalabel{smooth-singular}
Let $G\in\R^{n\times k}$ be any matrix with $\|G\|\le 1$
and let $V$ be a $n-k$ dimensional subspace with orthogonal projection $P_V.$
Let $H\sim \mathrm{N}(0,\tau^2/n)^{n\times k}$ be a random Gaussian matrix. 
Assume $k = o(n/\log n).$ 
Then, with probability $1-\exp(-\Omega(n)),$ we have
$\sigma_k\left( P_V (G+H) \right) \ge \Omega(\tau)\mper$
\end{lemma}

The proof follows from standard concentration arguments and is contained in
\sectionref{smooth-appendix}. To use this lemma in our context we'll introduce
a variant of $\mu$-coherence that applies to matrices rather than subspaces.

\begin{definition}[$\rho$-coherence]
\definitionlabel{rho-coherence}
Given a matrix $G\in\R^{n\times k}$ 
we let
$\rho(G) \defeq \frac nk\max_{i\in[n]}\|e_i^\trans G\|^2\mper$
\end{definition}

The next lemma is our main technical tool in this section. It shows that
adding a Gaussian noise term leads to a bound on the coherence after applying
the QR-factorization.

\begin{lemma}\lemmalabel{smoothgs}
Let $k= o(n/\log n)$ and $\tau\in(0,1).$
Let $U\in\R^{n\times k}$ be an orthonormal matrix.
Let $G\in\R^{n\times k}$ be a matrix such that $\|G\|\le1.$
Let $H\sim \mathrm{N}(0,\tau^2/n)^{k\times n}$ be a random Gaussian matrix. 
Then, with probability $1-\exp(-\Omega(n))-n^{-5},$ there is an orthonormal 
matrix $Q\in \R^{n\times 2k}$ such that:
\begin{enumerate}
\item $\range(Q)=\range([U \mid G+H])$ where $\range(Q)$ denotes the range
of~$Q,$
\item $\mu(Q) \le 
O\left(\frac1{\tau^2}\cdot\left(\rho(G)
+ \mu(U) + \log n \right)\right).$
\end{enumerate}
\end{lemma}

\begin{proof}
First note that 
$\range([U\mid G+H]) = \range([U \mid (I-UU^\trans)(G+H)]).$
Let $B=(I-UU^\trans)(G+H).$ Applying the QR-factorization to $[U \mid B],$ we
can find two orthonormal matrices $Q_1,Q_2\in\R^{n\times k}$ such that
have that $[Q_1 \mid Q_2 ] = [ U \mid BR^{-1} ]$ where $R\in\R^{k\times k}.$ 
That is $Q_1=U$ since $U$ is already orthonormal. Moreover, the columns of 
$B$ are orthogonal to $U$ and therefore we can apply the QR-factorization to
$U$ and $B$ independently.
We can now apply \lemmaref{smooth-singular} to the $(n-k)$-dimensional subspace $U^\bot$ 
and the matrix $G+H.$ It follows that with probability
$1-\exp(-\Omega(n)),$ we have $\sigma_k(B)\ge \Omega(\tau).$ Assume
that this event occurs.

Also, observe that $\sigma_k(B) = \sigma_k(R).$
The second condition is now easy to verify
\begin{align*}
\frac nk\Norm{e_i^\trans Q}^2
= \frac nk\Norm{e_i^\trans U}^2 
+ \frac nk\Norm{e_i^\trans BR^{-1}}^2 
= \mu(U) + \frac nk\Norm{e_i^\trans BR^{-1}}^2 
\end{align*}
On the other hand,
\[
\frac nk\Norm{e_i^\trans BR^{-1}}^2 
\le \frac nk\Norm{e_i^\trans B}^2\Norm{R^{-1}}^2 
\le O\left(\frac{n}{k\tau^2} \Norm{e_i^\trans B}^2\right)\mcom
\]
where we used the fact that $\Norm{R^{-1}}= 1/\sigma_k(R) = O(1/\tau).$
Moreover,
\[
\frac nk \Norm{e_i^\trans B}^2
\le 2\frac nk\Norm{e_i^\trans(I-UU^\trans)G}^2
+2\rho((I-UU^\trans) H)
\le 2\rho(G)+2\rho(UU^\trans G)+2\rho((I-UU^\trans) H)\mper
\]
Finally, $\rho(UU^\trans G) \le \mu(U)\|U^\trans G\|^2\le\mu(U)$
and, by \lemmaref{rho-gaussian}, we have  
$\rho((I-UU^\trans)H)\le O(\log n)$ with probability $1-1/n^{5}.$  
The lemma follows with a union bound over the failure probabilities.
\end{proof}

The next lemma states that when
\SmoothGS is invoked on an input of the form $AX + G$ with suitable parameters, the algorithm 
outputs a matrix of the form $X'=\text{\GS}(AX+G+H)$ whose coherence is bounded
in terms of $\mu(U)$ and $\rho(G)$ and moreover $H$ satisfies a bound on its norm.  
The lemma also permits to trade-off the amount of additional
noise introduced with the resulting coherence parameter. 
\begin{lemma}\lemmalabel{coherence-onestep}
\label{C0}
Let $\tau>0$ and assume $k=o(n/\log n).$ 
There is an absolute constant $C_{\ref{C0}}>0$ such that the following claim
holds. Let $G\in\mathbb{R}^{n\times k}.$ 
Let $X\in\R^{n\times k}$ be an orthonormal matrix 
such that $\nu\ge\max\Set{\|G\|,\|NX\|}.$
Assume that
\[
\mu \ge 
\frac{C_{\ref{C0}}}{\tau^2}\Big(\mu(U)
+ \frac{\rho(G)+\rho(NX)}{\nu^2} + \log n\Big)\mper
\] 
Then, for every $\epsilon \le \tau\nu$ satisfying $\log(n/\epsilon)\le n$ and
every $\mu\le n,$ we have with probability $1-O(n^{-4}),$ 
the algorithm $\text{\SmoothGS}(AX+G,\epsilon,\mu)$ 
terminates in $O(\log(n/\epsilon))$ steps and  
outputs $(X',H)$ such that $\mu(X')\le \mu$ and 
where $H$ satisfies $\|H\|\le \tau\nu.$
\end{lemma}
\begin{proof}
Suppose that \SmoothGS terminates in an iteration where $\sigma^2 \le
\tau^2\nu^2/4.$ We claim that in this case with probability
$1-\exp(-\Omega(n))$ we must have that $\|H\|\le\tau\nu.$ 
Indeed, assuming
the algorithm terminates when $\sigma^2 \le c\tau^2\nu^2/k,$ the algorithm
took at most $t=O(\log(n/\epsilon))\le O(n)$ steps. Further, let $H_1,\dots,H_t$
denote the random Gaussian matrices generated in each step. We claim that each
of them satisfies $\|H_t\|\le\tau\nu.$ Note that for all $t$ we have
$\E\|H_t\|^2\le\tau^2\nu^2/4$ since we assumed that the algorithm terminates
when $\sigma^2\le\tau^2\nu^2/4$ and therefore every $H_t$ has variance at most
$\sigma^2/n$ in each entry. The claim therefore follows directly from tail
bounds for the Frobenius norm of Gaussian random matrices and holds with
probability $1-\exp(-\Omega(n)).$ The next claim now finishes the proof.
\begin{claim}
With probability $1-O(1/n^4),$ the algorithm terminates in an iteration where
$\sigma^2 \le \tau^2\nu^2/4.$
\end{claim}
To prove the claim, consider the first iteration in which 
$\sigma^2 \ge \tau^2\nu^2/8.$ 
Let us define $G'= (NX + G)/2\nu.$ 
We can now apply \lemmaref{smoothgs} to the matrix $G'$
which satisfies the assumption of the lemma that $\|G'\|\le1.$  
The lemma then entails that with the stated
probability bound there is an orthonormal $n\times 2k$ matrix $Q$ such that
\[
\cR(Q)=\cR([ U \mid G' + H]) = \cR([U\mid G+NX + H])\mcom
\]
and moreover
$\mu(Q) \le 
O\left(\frac1{\tau^2}\cdot\left(\rho(G)
+ \mu(U) + \log n \right)\right)\mper$
On the one hand,
\[
\cR(X')= \cR(AX + G + H)
= \cR(MX + NX + G + H)
\subseteq \cR([U \mid NX + G + H])=\cR(V)\mper
\]
The inclusion follows from the fact that $U$ is an orthonormal basis for the
range of $M X=U\Sigma_UU^\trans X.$ 
On the other hand,
$\rho(G') = O\left(\rho(G/\nu) + \rho(NX/\nu')\right)\mper$
Hence, by \lemmaref{coherence-subspace} and the fact that
$\dim(Q)\le2\dim(X')$, we have
$\mu(X')\le 2\mu(Q) \le \mu\mper$
This shows that the termination criterion of the algorithm is satisfied
provided we pick $C_{\ref{C0}}$ large enough.
\end{proof}

\section{Convergence bounds for alternating minimization}
\sectionlabel{convergence}

The total sample complexity we achieve is the sum of two terms. The first one
is used by the initialization step that we discuss in
\sectionref{init}. The second term specifies the sample requirements 
for iterating the least squares algorithm. 
It therefore makes sense to define the following two
quantities:
\[
p_{\mathrm{init}} = \frac{k^2\mustar\|A\|_F^2\log
n}{\gamma_k^2\sigma_k^2n}
\quad\text{and}\quad
p_{\mathrm{LS}} =
\frac{k\mustar(\|M\|_F^2+\|N\|_F^2/\epsilon^2)\log(n/\epsilon)\log^2
n}{\gamma_k^5\sigma_k^2 n} 
\]
While the first term has a quadratic dependence on~$k$ it does not depend on
$\epsilon$ at all and it has single logarithmic factor. The second term
features a linear dependence on~$k.$ Our main theorem shows that if the
sampling probability is larger than the sum of these two terms, the algorithm
converges rapidly to the true unknown matrix.
\begin{theorem}[Main]
\theoremlabel{main}
Let $k,\epsilon>0.$ 
Let $A=M+N$ be a symmetric $n\times n$ matrix where $M$ is a matrix of
rank~$k$ with the spectral decomposition $M = U\Lambda_U U^\trans$ and
$N=(I-UU^\trans)A=V\Lambda_VV^\trans$ satisfies \eqref{noise-bound}. 
Let $\gamma_k = 1 - \sigma_{k+1}/\sigma_k$ where $\sigma_k$ is the smallest
singular value of~$M$ and $\sigma_{k+1}$ is the largest singular value of~$N.$
Assume each entry of $\Omega$ is included independently with
probability~$p$ satisfying
$p\gtrsim p_{\mathrm{init}}+p_{\mathrm{LS}}$.

Then, there are parameters $\mu= \Theta(\gamma_k^{-2}k(\mustar +\log n))$ and 
$L=\Theta(\gamma_k^{-1}\log(n/\epsilon))$ such that 
the output~$(X,Y)$ of
$\text{\SAltLS}(P_\Omega(A),\Omega,k,L,\epsilon,\mu)$ 
satisfies $\|(I-UU^\trans)X_L\| \le \epsilon$ with probability $9/10.$
\end{theorem}

Before we prove the theorem in \sectionref{proof-main}, we will state an
immediate corollary that gives bounds on the reconstruction error in 
the Frobenius norm.

\begin{corollary}[Reconstruction error]
\corollarylabel{reconstruction}
Under the assumptions of \theoremref{main}, we have that the output~$(X,Y)$ of
\SAltLS satisfies 
$\Norm{M - XY^\trans}_F \le \epsilon\Norm{A}_F$
with probability $9/10.$ 
\end{corollary}

\begin{proof}
Let $(X,Y)$ be the matrices given by our algorithm when invoked with
error parameter~$\epsilon/2.$
By \theoremref{main} we have 
$\Norm{UU^\trans - XX^\trans}
= \Norm{(I-UU^\trans)X}\le\frac{\epsilon}2\mper$
Using the proof of \theoremref{main} we also know that $Y = AX + G$ 
where $G$ is $(\epsilon/4)$-admissible so that $\|G\|_F \le \epsilon\sigma_k/2.$ 
Consequently,
\begin{align*}
\Norm{M-XY^\trans}_F
 = \Norm{M - XX^\trans A + XG}_F
&\le \Norm{UU^\trans A - XX^\trans A}_F + \Norm{XG}_F\\
&\le \Norm{UU^\trans - XX^\trans}\Norm{A}_F + \Norm{G}_F\\
&\le (\epsilon/2)\Norm{A}_F + (\epsilon/2)\sigma_k
\le \epsilon\Norm{A}_F \mper
\end{align*}
In the second inequality we used that for all matrices $P,Q$ we
have $\|PQ\|_F\le\|P\|\cdot\|Q\|_F.$
\end{proof}

\subsection{Proof of \theoremref{main}}
\sectionlabel{proof-main}
\begin{proof}
We first apply \theoremref{init} (shown below) to conclude that with
probability $19/20,$ the initial matrix~$X_0$ satisfies
$\Norm{V^\trans X_0}\le 1/4$ and $\mu(X_0)\le 32\mu(U)\log n.$ 
Assume that this event occurs.
Our goal is now to apply \theoremref{convergence}. Consider the sequence of
matrices $\Set{(X_{\ell-1},\tilde G_\ell)}_{\ell=1}^L$ obtained by the execution of
\SAltLS starting from $X_0$ and letting $\tilde G_\ell=G_\ell + H_\ell$
where $G_\ell$ is the error term corresponding to the $\ell$-step of \MedianLS,
and $H_\ell$ is the error term introduced by the application of \SmoothQR at
step~$\ell.$
To apply \theoremref{convergence},
we need to show that this sequence of matrices is $(\epsilon/2)$-admissible
for \NSI with probability $19/20$. The theorem then directly gives that
$\|V^\trans X_L\|\le \epsilon$ and this would conclude our proof by summing
up the error probabilities.

Let
\[
\tau = \frac{\gamma_k}{128}
\qquad\text{and}\qquad 
\hat\mu = \frac{C_{\ref{C0}}}{\tau^2}\left(20\mustar + \log n\right)\mper
\]
Let $\mu$ be any number satisfying $\mu\ge\hat\mu.$ Since
$\hat\mu=\Theta(\gamma_k^{-2}k(\mustar+\log n)),$ this satisfies the
requirement in the theorem.
We prove that with probability $9/20,$ the following three claims hold:
\begin{enumerate}
\item $\Set{(X_{\ell-1},G_\ell)}_{\ell=1}^L$ is
$(\epsilon/4)$-admissible,
\item $\Set{(X_{\ell-1},H_\ell)}_{\ell=1}^L$ is
$(\epsilon/4)$-admissible,
\item for all $\ell\in\{0,\dots,L-1\},$ we have $\mu(X_{\ell})\le\mu.$
\end{enumerate}
This implies the claim that we want using a triangle inequality since $\tilde
G_\ell = G_\ell + H_\ell.$

The proof of these three claims is by mutual induction. For $\ell=0,$ we only
need to check the third claim which follows form the fact that $X_0$ satisfies
the coherence bound. Now assume that all three claims hold at step $\ell-1,$
we will argue that the with probability $1-n/100,$ all three claims hold at
step $\ell.$ Since $L\le n,$ this is sufficient.

The first claim follows from Lemma~\ref*{lem:median-bound}
using the induction
hypothesis that $\mu(X_{\ell-1})\le\hat\mu.$ 
Specifically, we apply the lemma with $\delta= c
\min\{\gamma_k\sigma_k/\|M\|_F,\epsilon\gamma_k\sigma_k/\|N\|_F\}$ for
sufficiently small constant $c>0.$ 
The lemma requires the lower bound
$p \gtrsim \frac{k\mustar\log^2 n}{\delta^2 n}.$ We can easily verify that the right
hand side is a factor $L=\Theta(\gamma_k^{-1}\log(n/\epsilon))$ smaller than
what is provided by the assumption of the theorem. This is because new samples
are used in each of the $L$ steps so that we need to divide the given bound
by~$L.$
Lemma~\ref*{lem:median-bound}
now gives with probability $1-1/n^3$ the upper bound 
\[
\|G_\ell\|_F\le \frac14\left(\frac1{32}\gamma_k\sigma_k\Norm{V^\trans
X_{\ell-1}} + \frac{\epsilon}{32}\gamma_k\sigma_k\right)\mper
\]
In particular, this satisfies the definition of $\epsilon/4$-admissibility.
We proceed assuming that this event occurs as the error probability is small
enough to ignore.

The remaining two claims follow from \lemmaref{coherence-onestep}. We will apply
the lemma to $AX_{\ell} + G_\ell$ 
with $\nu = \sigma_k(\Norm{V^\trans X_{\ell-1}}+\epsilon)$ and $\tau$
as above.  Note that 
\[
\Norm{N X_{\ell-1}} \le \sigma_k\Norm{V^\trans X_{\ell-1}}\mper
\]
Hence we have $\nu\ge\max\{\|G_\ell\|,\|NX_{\ell-1}\|\}$ as required by the lemma. 
The lemma also requires a lower bound~$\mu.$
To satisfy the lower bound we invoke 
\lemmaref{rho-bound} showing that with probability $1-1/n^2,$ we have
\[
\frac1{\nu^2}\left(\rho(G)+\rho(NX)\right)
\le 10\mustar.
\]
We remark that this is the lemma that uses the assumption on $N$ provided by
\eqref{noise-bound}.
Again we assume this event occurs. In this case we have
\[
\mu \ge \hat\mu = \frac{C_{\ref{C0}}}{\tau^2}\left(20\mustar + \log n\right)
\]
and so we see that $\mu$ satisfies the requirement of
\lemmaref{coherence-onestep}.  It follows that \SmoothGS produces with
probability $1-1/n^4$ a matrix $H_\ell$ such that 
\[
\|H_\ell\|\le \tau\nu
\le \frac{\gamma_k\nu}{128}
\le \frac14\left(\frac1{32}\gamma_k\sigma_k\Norm{V^\trans
X_{\ell-1}}_F + \frac{\epsilon}{32}\gamma_k\sigma_k\right)\mper
\]
In particular, $H_\ell$ satisfies the requirement of
$(\epsilon/4)$-admissibility. Moreover, the lemma gives that 
$\mu(X_\ell)\le \mu.$ This shows that also the second and third claim 
of our inductive claim continue to hold.
All error probabilities we incurred were $o(1/n)$ and we can sum up the error
probabilities over all $L\le n$ steps to concludes the proof of the theorem. 
\end{proof}

\section{Finding a good starting point} 
\sectionlabel{init}
\figureref{init} describes an algorithm that 
computes the top $k$ singular vectors of $P_{\Omega}(A)$ and truncates them in
order to ensure incoherence. The algorithm serves as a fast
initialization procedure for our main algorithm. This general approach is
relatively standard in the literature. However, our truncation argument differs from
previous approaches. Specifically, we use a random orthonormal transformation
to spread out the entries of the singular vectors before truncation. This
leads to a tighter bound on the coherence.

\begin{figure}[ht]
\begin{boxedminipage}{\textwidth}
\noindent \textbf{Input:} 
Target dimension $k,$ observed set of indices
$\Omega\subseteq [n]\times[n]$ of an unknown symmetric matrix $A\in\R^{n\times
n}$ with entries $P_\Omega(A),$ coherence parameter~$\mu\in\R.$

\noindent \textbf{Algorithm} $\text{\Init}(P_\Omega(A),\Omega,k,\mu):$
\begin{enum}
\item Compute the first $k$ singular vectors~$W\in\R^{n\times k}$ of $P_\Omega(A).$
\item $\tilde W \leftarrow WO$ where $O\in\R^{k\times k}$ is a random orthonormal matrix.
\item $T \leftarrow {\cal T}_{\mu'}(\tilde W)$ with
$\mu'=\sqrt{8\mu\log(n)/n}$ where ${\cal T}_c$ replaces each entry of its
input with the nearest number in the interval~$[-c,c].$
\item $X \leftarrow \QR(T)$ 
\end{enum}
\noindent \textbf{Output:} Orthonormal matrix $X\in\R^{n\times k}.$
\end{boxedminipage}
\caption{Initialization Procedure (\Init)}
\figurelabel{init}
\end{figure}

\begin{theorem}[Initialization]
\theoremlabel{init}
Let $A=M+N$ be a symmetric $n\times n$ matrix where $M$ is a matrix of
rank~$k$ with the spectral decomposition $M = U\Lambda_U U^\trans$ and
$N=(I-UU^\trans)A$ satisfies \eqref{noise-bound}. Assume that each entry
is included in $\Omega$ independently probability
\begin{equation}\equationlabel{init-sample-size}
p \ge \frac{C k (k\mu(U)+\mu_N) (\|A\|_F/\gamma_k\sigma_k)^2 \log n}{n}
\end{equation}
for a sufficiently large constant $C>0.$ Then, the algorithm \Init returns an
orthonormal matrix $X\in \R^{n\times k}$ such that with probability $9/10,$ 
$\|V^\trans X\|_F\le 1/4$ and $\mu(X)\le 32\mu(U)\log n.$
\end{theorem}
\begin{proof}
The proof follows directly from \lemmaref{init-1} and \lemmaref{init-2} 
below.
\end{proof}

\begin{remark}
To implement \Init it is sufficient to compute an approximate singular value
decomposition of~$P_{\Omega}(A).$ From our analysis it is easy to see that it
is sufficient to compute the $k$-th singular value to accuracy, say,
$\gamma_k\sigma_k/100k.$ This can be done efficiently using, for example, the
Power Method (Subspace Iteration) with $O(k\gamma^{-1})\log n)$ iterations.
See~\cite{Higham,Stewart} for details on the Power Method. In particular,
the running time of this step is dominated by the running time of \AltLS.
\end{remark}

\begin{lemma}
\lemmalabel{init-1}
Assume that $\Omega$ satisfies \equationref{init-sample-size}. Then, 
$\Pr\Set{\|V^\trans W\|_2 \le 1/16\sqrt{k}}\ge1-1/n^2.$ 
\end{lemma}
\begin{proof}
By our assumption on~$A,$ we have $\max_{i\in[n]}\Norm{e_i^\trans N}^2 \le
(\mu_N /n)\|A\|_F^2$ and $\max_{i,j\in[n]}|N_{ij}|\le (\mu_N/n)\|A\|_F.$ Moreover,
$\max_i \Norm{e_i^\trans M}^2\le (\mu(U) k/n)\|M\|_F^2$ and $\max_{i,j}
|M_{ij}|\le (\mu(U) k/n)\|M\|_F.$ This shows that 
\[
\max_i \Norm{e_i^\trans A}^2\le\frac{\mu(U) k+\mu_N}n\|A\|_F^2
\qquad\text{and}\qquad 
\max_{i,j} |A_{ij}|\le \frac{\mu(U) k+\mu_N}n\|A\|_F\mper
\]
Plugging these upper bounds into \lemmaref{init-tail} together with our sample 
bound in
\equationref{init-sample-size}, we get that
\[
\Pr\Set{\Norm{A-P_\Omega(A)}>\frac{\gamma_k\sigma_k}{32\sqrt{k}}}\le
1/n^2\mper
\]
Put $\epsilon = \gamma_k\sigma_k/32\sqrt{k}.$
Now assume that $\Norm{A-P_\Omega(A)}\le\epsilon$ and let $W$ be the top $k$
singular vectors of $P_\Omega(A).$ On the one hand,
$\sigma_{k}(P_\Omega(A))\ge \sigma_{k}(A) - \epsilon
\ge \sigma_k-\gamma_k\sigma_k/2\mper$
One the other hand, by definition, $\sigma_{k+1}(A) = \sigma_k-\gamma_k\sigma_k\mper$
Hence, by the Davis-Kahan $\sin\theta$-theorem~\cite{DavisK70,Stewart} we have that
\[
\|V^\trans W\| = \sin\theta_k(U,W)\le
\frac{\epsilon}{\sigma_k(P_\Omega(A))-\sigma_{k+1}(A)}
\le \frac{2\epsilon}{\gamma_k \sigma_k}
= \frac1{16\sqrt{k}}\mper\qedhere
\]
\end{proof}

\begin{lemma}
\lemmalabel{init-2}
Assume that $\|V^\trans W\|_2\le1/16\sqrt{k}.$ Then, with probability $99/100$
we have
$\Norm{V^\trans X}_F\le 1/4$
and $\mu(X) \le 32\mu(U)\log n.$
\end{lemma}
\begin{proof}
By our assumption on $W,$ there exists an orthonormal transformation
$Q\in\R^{k\times k}$ such that
$\Norm{UQ - W}_F \le 1/16\mper$
Moreover, $\mu(UQ)=\mu(U)\le\mu.$ In other words, $W$ is close in Frobenius
norm to an orthonormal basis of small coherence. A priori it could be that
some entries of $UQ$ are as large as $\sqrt{\mu k/n}.$ However, after rotating
$UQ$ be a random rotation, all entries will be as small as $\sqrt{\mu
\log(n)/n}.$ This is formalized in the next claim.
\begin{claim}\claimlabel{random-rotation}
Let $Y\in\R^{n\times k}$ be any orthonormal basis with $\mu(Y)\le\mu.$ Then,
for a random orthonormal matrix $O\in\R^{k\times k},$ we have
$\Pr\Set{\max_{ij}|(YO)_{ij}| > \sqrt{8\mu\log(n)/n}}\le \frac1{n^2}\mper$
\end{claim}
\begin{proof}
Consider a single entry $Z=(YO)_{ij}.$ 
Observe that $Z$ is distributed like a coordinate of a random vector in $\R^k$
of norm at most~$\sqrt{\mu k/n}.$ By measure concentration, we have
\[
\Pr\Set{|Z|>\epsilon \sqrt{\mu k/n}}\le 4\exp(-\epsilon^2 k/2)\mper
\]
This follows from Levy's Lemma~(see \cite{Matousek02}) using the fact that 
projection onto a single coordinate in $\R^k$ is a Lipschitz function 
on the $(k-1)$-dimensional sphere. The median of this function is $0$ due to
spherical symmetry. Hence, the above bound follows. 
Putting $\epsilon = \sqrt{8\log(n)/k},$ we have that 
\[
\Pr\Set{|Z|>\sqrt{8\mu \log(n)/n}}\le 4\exp(3\log(n))= 4n^{-4}\mper
\]
Taking a union bound over all $kn\le n^2/4$ entries, we have that
with probability $1-1/n^2,$
\[
\max_{i,j} |(YO)_{ij}| \le \sqrt{8\mu\log(n)/n}\mper\qedhere
\]
\end{proof}
Applying the previous claim to $UQ,$ we have that 
with probability $1-1/n^2,$ for all $i,j,$ $(UQO)_{ij}\le \mu'.$ Furthermore,
because a rotation does not increase Frobenius norm, we also have
$\Norm{UQO-WO}_F\le1/16.$ Truncating the entries of $WO$ to $\mu'$ can
therefore only decrease the distance in Frobenius norm to $UQO.$
Hence,
$\Norm{UQO-T}_F\le 1/16\mper$
Also, since truncation is a projection onto the set $\{B\colon
|B_{ij}|\le\mu'\}$ with respect to Frobenius norm, we have
\[
\Norm{WO-T}_F
\le\Norm{UQO-T}_F
\le \frac1{16}\mper
\]
We can write $X = TR^{-1}$ where $R$ is an invertible linear
transformation with the same singular values as $T$ and thus satisfies
\[
\|R^{-1}\|
= \frac{1}{\sigma_k(T)}
\le \frac{1}{\sigma_k(WO)-\sigma_1(WO-T)}
\le \frac{1}{1-1/16}\le 2\mper
\]
Therefore,
\[
\Norm{e_i^\trans X}
= \Norm{e_i^\trans TR^{-1}}
\le \Norm{e_i^\trans T}\Norm{R^{-1}}
\le 2\Norm{e_i^\trans T}
\le 2\sqrt{8 k\mu(U)\log(n)/n}\mper
\]
Hence, 
\[
\mu(X) \le \frac nk\cdot \frac{32 k\mu(U)\log(n)}n\le 32\mu(U)\log(n)\mper
\] 
Finally,
\begin{align*}
\Norm{V^\trans X}_F
= \Norm{V^\trans TR^{-1}}_F
& \le \Norm{V^\trans T}_F\Norm{R^{-1}} 
 \le 2\Norm{V^\trans T}_F \\
& \le 2\Norm{V^\trans WO}_F + 2\Norm{WO-T}_F
\le 2\Norm{V^\trans W}_F + \frac{1}{8}
\le \frac{1}{4}\mper\qedhere
\end{align*}
\end{proof}

\section*{Acknowledgments}
Thanks to David Gleich, 
Prateek Jain,
Jonathan Kelner, 
Raghu Meka, 
Ankur Moitra, 
Nikhil Srivastava, 
and Mary Wootters 
for many helpful discussions.
We thank the Simons Institute for Theoretical Computer Science at Berkeley,
where some of this research was done.

\paragraph{Changes from the arxiv v3.}
Thanks to Xiaowei Zhong for pointing out a bug in Lemma~A.5 in the
previous version of this manuscript. The original proof expanded using $P$
directly. Although its diagonal entries are independent, they have mean one,
so independence does not eliminate the cross terms. The correction uses the
orthogonality relation $z^\trans X=0$ to write
$z^\trans PX=z^\trans(P-I)X$. The diagonal entries of $P-I$ are independent
and mean zero, so the cross terms vanish and the claimed bound follows.

\bibliographystyle{moritz}
\bibliography{moritz}

\appendix

\section{Large deviation bounds}
\sectionlabel{concentration}

We need some matrix concentration inequalities.
Turn to~\cite{Tropp12} for background.
\begin{theorem}[Matrix Bernstein]
\theoremlabel{matrix-bernstein}
Consider a finite sequence $\Set{Z_k}$ of independent random matrices with
dimensions $d_1\times d_2.$ Assume that each random matrix satisfies
$\E Z_k = 0$ and $\Norm{Z_k}\le R$ almost surely.
Define
$\sigma^2 \defeq \max\Set{
\Norm{\textstyle\sum_k \E Z_kZ_k^\trans},
\Norm{\textstyle\sum_k \E Z_k^\trans Z_k} }\mper$
Then, for all $t\ge 0,$
\[
\Pr\Set{\Norm{\textstyle\sum_k Z_k}\ge t }
\le (d_1 + d_2 )\cdot \exp\left(\frac{-t^2/2}{\sigma^2 + Rt/3}\right)\mper
\]
\end{theorem}
\begin{theorem}[Matrix Chernoff]
\theoremlabel{matrix-chernoff}
Consider a finite sequence~$\Set{X_k}$ of independent self-adjoint matrices of
dimension~$d.$ Assume that each random matrix satisfies
$X_k\succeq 0$ and
$\lambda_{\mathrm{max}}(X_k)\le R$ almost surely.
Define 
$\mu_{\mathrm{min}}\defeq \lambda_{\mathrm{min}}\left(\sum\nolimits_k\E
X_k\right)\mper$
Then,
\[
\Pr\Set{ \lambda_{\mathrm{min}}\left(\sum\nolimits_k X_k\right)
\le (1-\delta) \mu_{\mathrm{min}} }
\le d\cdot
\exp\left(\frac{-\delta^2\mu_{\mathrm{min}}}{2R}\right)
\]
\end{theorem}

\subsection{Error bounds for initialization}

\begin{lemma}
\lemmalabel{init-tail}
Suppose that $A \in \R^{m\times n}$ and let $\Omega \subset [m] \times [n]$ be a 
random subset where each entry is included independently with probability~$p$.  Then
\[\Pr\Set{ \Norm{ P_{\Omega}(A) - A} > u } \leq n
\exp\left(\frac{-u^2/2 }{\sigma^2
+ \frac{u}{3}(1/p-1)\max_{ij}|A_{ij}| }\right).\]
where $\sigma^2=(1/p-1)\max\Set{\max_i \Norm{e_i^\trans A}^2,\max_j \Norm{Ae_j}^2}.$
\end{lemma}

\begin{proof}
Let $\xi_{ij}$ be independent Bernoulli-$p$ random variables, 
which are $1$ if $(i,j) \in \Omega$ and $0$ otherwise. 
Consider the sum of 
independent random matrices
$P_\Omega(A) - A = \sum_{i,j} \left(\frac{\xi_{ij}}{p} - 1\right) A_{ij}
e_ie_j^\trans\mper$
Applying \theoremref{matrix-bernstein}, we conclude that
\[
\Pr\Set{ \Norm{ P_{\Omega}(A) - A} > u } \leq n \exp\left(\frac{-u^2/2 }{
\sigma^2 + Ru/3 }\right)\mper
\]
Here we use that
\[\textstyle
\Norm{
\E 
\sum_{i,j} 
\left(\frac{\xi_{ij}}{p} - 1\right)^2
A_{ij}^2 e_ie_j^\trans e_j{e_i}^\trans
} 
= 
\left(\frac{1}{p} -1 \right) 
\max_i \Norm{e_i^\trans A}^2
\]
and similarly
\[
\Norm{\E \sum\nolimits_{i,j} \left(\frac{\xi_{ij}}{p} - 1\right)^2
A_{ij}^2 e_je_i^\trans e_i e_j^\trans} = \Big(\frac{1}{p} -1 \Big) \max_i
\Norm{Ae_j}^2\mper
\]
Further,
$\Norm{\left(\frac{\xi_{ij}}{p} - 1\right) A_{ij} e_ie_j^\trans} \leq R =
\left(\frac{1}{p} -1\right)\max_{ij}|A_{ij}|\mper$
This concludes the proof.
\end{proof}

\subsection{Error bounds for least squares}
\sectionlabel{inverse-bound}

\begin{lemma}
\lemmalabel{inverse-bound}
\lemmalabel{B-bound}
Let $0<\delta<1$ and let $i\in[n].$ 
Assume that 
$p \gtrsim \frac{k\mu(X)\log n}{\delta^2n}.$
Then, 
$\Pr\Set{\Norm{B_i^{-1}}\ge \frac1{1-\delta}}\le \frac1{n^5}\mper$
\end{lemma}
\begin{proof}
Let $B=B_i$ and $p=p_\ell.$
Clearly, $\Norm{B^{-1}}=1/\lambda_{\mathrm{min}}(B).$ We will use
\theoremref{matrix-chernoff} to lower bound the smallest eigenvalue of $B$ by
$1-\delta.$ Denoting the rows of $X$ by $x_1,\dots,x_n\in\R^k$ we
have
$B = \sum_{i=1}^n \frac1p Z_i x_ix_i^\trans\mcom$
where $\{Z_i\}$ are independent $\mathrm{Bernoulli}(p)$ random variables.
Moreover,
$\E B = X^\trans X = \mathrm{Id}_{k\times k}\mper$
Therefore, in the notation of \theoremref{matrix-chernoff}, this is
$\mu_{\mathrm{min}}(B)=1.$ Moreover, using our lower bound on $p,$
$\Norm{\frac1p Z_i x_ix_i^\trans}
\le \frac1p\|x_i\|^2 
\le \frac{\mu(X_{\ell-1})k}{pn}
\le \frac {\delta^2}{20\log n}\mper$
Hence, by \theoremref{matrix-chernoff},
$\Pr\Set{\lambda_{\mathrm{min}}(B)\le 1-\delta}\le k\exp\left(-10\log
n\right)\mper$ The claim follows.
\end{proof}

\begin{lemma}
\lemmalabel{E-bound}
\lemmalabel{M-row-bound}
\lemmalabel{row-sampling}
Let $X\in\R^{n\times k}$ be a fixed orthonormal matrix, and let
$z\in\R^n$ be fixed with $z^\trans X=0$.  Let
$P=p^{-1}\operatorname{diag}(\xi_1,\dots,\xi_n)$, where the $\xi_j$ are
independent $\operatorname{Bernoulli}(p)$ variables.  If $0<\eta<1$ and
\[
p\ge \frac{10k\mu(X)}{\eta^2n},
\]
then
\[
\Pr\Set{\Norm{z^\trans PX}>\eta\Norm{z}}\le\frac1{10}.
\]
\end{lemma}

\begin{proof}
Put $D=P-I$, let $d_j=D_{jj}=\xi_j/p-1$, and write
$x_j^\trans=e_j^\trans X$. Since $z^\trans X=0$,
\[
z^\trans PX=z^\trans DX=\sum_{j=1}^n z_jd_jx_j^\trans.
\]
The variables $d_1,\dots,d_n$ are independent and satisfy
\[
\E d_j=0
\qquad\text{and}\qquad
\E d_j^2=\frac{1-p}{p}.
\]
In particular, if $j\ne s$, then
$\E[d_jd_s]=\E[d_j]\E[d_s]=0$. Therefore,
\begin{align*}
\E\Norm{z^\trans PX}^2
&=\sum_{j=1}^n\sum_{s=1}^n
  z_jz_s\E[d_jd_s]\langle x_j,x_s\rangle\\
&=\frac{1-p}{p}\sum_{j=1}^n z_j^2\Norm{x_j}^2\\
&\le \frac{k\mu(X)}{pn}\Norm{z}^2.
\end{align*}
Markov's inequality now gives
\[
\Pr\Set{\Norm{z^\trans PX}>\eta\Norm{z}}
\le \frac{k\mu(X)}{pn\eta^2}\le\frac1{10}\mper\qedhere
\]
\end{proof}

\section{Additional lemmas and proofs for smooth QR factorization}
\sectionlabel{smooth-appendix}

\begin{proof}[Proof of \lemmaref{smooth-singular}]
Fix a unit vector $x\in\R^k.$ We have
\[
\|P_V(G+H)x\|^2 \ge \|P_V H x\|^2 - |\langle P_V G x, P_V H x\rangle|
\]
Note that $g=Hx$ is distributed like $\mathrm{N}(0,\tau^2/n)^n$ and $y=P_VCx$ has norm at
most~$1.$ Due to the rotational invariance of the Gaussian measure, we may
assume without loss of generality that $V$ is the subspace spanned by
the first $n-k$ standard basis vectors in $\R^n.$ 
Hence, denoting $h\sim \mathrm{N}(0,\tau^2/n)^{n-k},$ our goal is to lower bound
$\|h\|^2 - |\langle y,h\rangle|.$ Note that $\E\|h\|^2\ge \tau^2/2$ and by
standard concentration bounds for the norm of a Gaussian variable we have
\[
\Pr\Set{\|h\|^2 \le \tau^2/4} \le \exp(-\Omega(n))\mper
\]
On the other hand $\langle y,h \rangle$ is distributed like a one-dimensional
Gaussian variable of variance at most $\tau^2/n.$ Hence, by Gaussian tail
bounds,
$\Pr\Set{\langle y,h\rangle^2 > \tau^2/8 }\le\exp(-\Omega(n))\mper$
Hence, with probability $1-\exp(-\Omega(n)),$ we have $\|P_V(G+H)x\|\ge
\Omega(\tau).$
We can now take a union bound over a net of the unit sphere in $\R^k$ 
of size $\exp(O(k\log k))$ to conclude that 
with probability $1-\exp(O(k\log k))\exp(-\Omega(n)),$ we have
for all unit vectors $x\in\R^k$ that
$\|P_V(G+H)x\|\ge \Omega(\tau)\mper$
Therefore $\sigma_k(P_V(G+H))\ge \Omega(\tau).$
By our assumption $\exp(O(k\log k)) = \exp(o(n))$ and hence this event occurs
with probability $1-\exp(-\Omega(n)).$
\end{proof}

\begin{lemma}\lemmalabel{rho-gaussian}
Let $P$ be the projection onto an $(n-k)$-dimensional subspace.
Let $H\sim \mathrm{N}(0,1/n)^{n\times k}.$ Then, $\rho(PH)\le O(\log n)$ with
probability $1-1/n^5.$
\end{lemma}

\begin{proof}
We have that $P = (I-UU^\trans)$ for some $k$-dimensional basis $U.$
Hence,
\[
\rho(PU) \le O(\rho(H)) + O(\rho(UU^\trans H))\mper
\]
Using concentration bounds for the norm of each row of $H$ and a union
bound over all rows it follows straightforwardly that $\rho(H) \le O(\log n)$
with probability $1-1/2n^5.$
The second term satisfies
\[
\rho(UU^\trans)\le \rho(U)\|U^\trans H\|^2
= \mu(U)\|U^\trans H\|^2\mper
\]
But $U^\trans H$ is a Gaussian matrix $\mathrm{N}(0,1/n)^{k\times k}$ and hence its
largest singular value satisfies $\|U^\trans H\|^2\le O(k\log(n)/n)$ with 
probability $1-1/2n^5.$
\end{proof}

\begin{lemma}\lemmalabel{coherence-subspace}
Let $X,Y$ be $k$ and $k'$ dimensional subspaces, respectively, such that
$\cR(X)\subseteq\cR(Y).$ Then, $\mu(X) \le \frac{k'}{k}\mu(Y)\mper$
\end{lemma}
\begin{proof}
We know that $\mu(Y)$ is rotationally invariant. Therefore, without loss of
generality we may assume that $Y=[X\mid X']$ for some orthonormal matrix $X'.$
Here, we identify $X$ and $Y$ with orthonormal bases. Hence,
\[
\mu(X)
= \frac{n}{k}\max_{i\in[n]}\|e_i^\trans X\|^2
\le \frac{n}{k}\max_{i\in[n]}\left(\|e_i^\trans X\|^2 + \|e_i^\trans X'\|^2\right)
= \frac{n}{k}\max_{i\in[n]}\|e_i^\trans Y\|^2
= \frac{k'}{k}\mu(Y)\mper\qedhere
\]
\end{proof}

The following technical lemma was needed in the proof of \theoremref{main}.

\begin{lemma}\lemmalabel{rho-bound}
Under the assumptions of \theoremref{main}, 
we have for every $\ell \in [L]$
and $\nu = \frac{\sigma_k}{32}(\Norm{V^\trans X_{\ell-1}} + \epsilon)$
with probability $1-1/n^2,$
\[
\frac1{\nu^2}\left( \rho(G)+\rho(NX_{\ell-1})\right) \le 3\mustar\mper
\]
\end{lemma}
\begin{proof}
Given the lower bound on $p$ in \theoremref{main} 
we can apply \lemmaref{medianls-error} to conclude that 
$\|e_i^\trans G_\ell^M\|\le\sqrt{k\mu(U)/n} \cdot \nu$
and $\|e_i^\trans G_\ell^N\|\le \sqrt{\mustar/n}\cdot \nu.$ 
Hence, $\rho(G_\ell)/\nu^2\le \mustar.$

Further, we claim that
$\|e_i^\trans NX\|^2\le (\mustar /n)\sigma_k\|V^\trans U\|$ for all
$i\in[n],$ because
\[
\Norm{e_i^\trans NX}
\le\Norm{e_i^\trans V\Sigma_V}\cdot\Norm{V^\trans X_{\ell-1}}
=\Norm{e_i^\trans N}\cdot\Norm{V^\trans X_{\ell-1}}\mper
\]
Here we used the fact that 
\[
\Norm{e_i^\trans N}^2 
= \Norm{e_i^\trans N V}^2
+\Norm{e_i^\trans N U}^2  
= \Norm{e_i^\trans N V}^2
= \Norm{e_i^\trans V\Sigma_V}^2\mper
\]
Using \equationref{noise-bound}, 
this shows that $\rho(NX_{\ell-1})/\nu^2\le\mustar$ and finishes the proof.
\end{proof}

\section{Splitting up the subsample}
\sectionlabel{split}
We needed a procedure $\text{\Split}(\Omega,t)$ that takes a sample $\Omega$ 
and splits it into $t$ independent samples that preserve the distributional
assumption that we need. The next lemma is standard.

\begin{lemma}\lemmalabel{split}
There is a procedure $\text{\Split}(\Omega,t)$ such that if $\Omega$ is
sampled by including each element independently with probability~$p,$ then
$\text{\Split}(\Omega,t)$ outputs independent random variables
$\Omega_1,\dots,\Omega_t$ such that each set $\Omega_i$ includes each element
independently with probability $p_i\ge p/t.$
\end{lemma}
\begin{proof}[Proof sketch]
Consider independent random samples $\Omega_1',\dots,\Omega_t'$ where each set
contains every element independently with probability $p/2t.$ Consider the
multi-set $\Omega'$ obtained from taking the union of these sets (counting
multiplicities). Each element occurs in $\Omega'$ at least once with
probability $p'=1-(1-p/t)^t\le p.$ The multiplicity is distributed according to
a binomial random variable. Hence, we can simulate the distribution of
$\Omega'$ given the random sample $\Omega$ by subsampling so that each entry
is included with probability $p'$ and then introducing multiplicities randomly
according to the correct Binomial distribution. On the other hand, given the
random variable $\Omega'$ we can easily simulate $\Omega_1',\dots,\Omega_t'$
by assigning each element present in $\Omega'$ with multiplicity $k$ to a
random subset of $k$ out $t$ sets.
\end{proof}

\section{Generalization to rectangular matrices}
\sectionlabel{rectangular}
For our purposes it will suffice to consider symmetric square matrices. 
This follows from  a simple transformation that preserves the
matrix coherence and singular vectors of the matrix. Indeed, given a matrix
$B\in\R^{m\times n}$ and $m\le n$ with singular value decomposition 
$B=\sum_{i=1}^r \sigma_i u_iv_i^\trans,$ we may consider the symmetric $(m+n)\times(m+n)$ 
matrix
$A = \left[\begin{array}{cc}
0 & B \\
B^\trans & 0
\end{array}
\right]\mper$
The matrix $A$ has the following properties:
$A$ has a rank $2\cdot\mathrm{rank}(B)$ and singular values
$\sigma_1(B),\dots,\sigma_r(B)$ each occurring with multiplicity two.
The singular vectors corresponding to a singular value $\sigma_i$ are spanned
by the vectors $\Set{(u_i,0), (0,v_i)}.$ 
In particular, an algorithm to find a rank~$2k$ approximation to $A$ also
finds a rank~$2k$ approximation to $B$ up to the same error.

Moreover, let $\tilde U$ denote the space spanned by the top $2k$ singular
vectors of $A,$ and let $U,$ respectively $V$, denote the space spanned by the
top~$k$ left, respectively right, singular vectors of $B.$ Then
$\mu(\tilde U) \le \frac{n+m}{2k}\left(\frac{\mu(U)k}{m} +
\frac{\mu(V)k}n\right)
\le \frac{n+m}{m}\max\Set{\mu(U),\mu(V)}\mper$
Note that we can assume that $(n+m)/m$ is constant by splitting $B$ into a
sequence of $m\times O(m)$ matrices and recovering each matrix separately.
It will also be important for us that we can turn a uniformly random subsample
of $B$ into a uniformly random subsample of $A.$ This is easily accomplished
by splitting the sample into two equally sized halves, using one for $B$ and
one for $B^\trans.$ The remaining quadrants of $A$ are $0$ and can be
subsampled trivially of any given density.

\end{document}